%% file: iclr2027_conference.tex
\documentclass{article}
\PassOptionsToPackage{table}{xcolor}
\usepackage{iclr2027_conference,times}

\input{math_commands.tex}

\usepackage{amsthm}
\usepackage{centernot}
\usepackage[table]{xcolor}
\usepackage{hyperref}
\usepackage{url}
\usepackage{graphicx}
\usepackage{wrapfig}
\usepackage{booktabs}
\usepackage{placeins}
\usepackage{multirow}
\usepackage{enumitem}
\usepackage{afterpage}

\definecolor{qwenSeven}{HTML}{527A9A}
\definecolor{qwenFourteen}{HTML}{7B7194}
\definecolor{gemmaFour}{HTML}{A47B58}
\definecolor{gemmaTwelve}{HTML}{708878}
\definecolor{tableHeader}{HTML}{F0F1F2}

\newtheorem{definition}{Definition}[section]

\newtheorem{proposition}[definition]{Proposition}

\newtheorem*{abrhstatement}{\textit{Answer-Basin Representation Hypothesis}}

\title{The Answer-Basin Representation Hypothesis: We Are Not Probing or Steering Concepts}

\author{%
\normalfont%
\textbf{Manjiang Yu}$^{1}$,
\textbf{Hongji Li}$^{2}$,
\textbf{Zihan Wang}$^{1}$,
\textbf{Junwei Chen}$^{3}$, \\
\textbf{Xue Li}$^{1}$,
\textbf{Priyanka Singh}$^{1}$,
\textbf{Yang Cao}$^{3}$,
\textbf{Lijie Hu}$^{2}$ \\
\normalfont $^{1}$The University of Queensland \\
\normalfont $^{2}$Mohamed bin Zayed University of Artificial Intelligence \\
\normalfont $^{3}$Institute of Science Tokyo
}

\iclrfinalcopy

\begin{document}

\maketitle
\begingroup
\renewcommand{\thefootnote}{}
\footnotetext{Code: \url{https://github.com/V1centNevwake/answer-basin-representation}.}
\endgroup
\lhead{Preprint}

\input{section/iclr/0_abstract.tex}
\input{section/iclr/1_intro.tex}
\input{section/iclr/2_hypothesis.tex}
\input{section/iclr/3_probing.tex}
\input{section/iclr/4_steering.tex}
\input{section/iclr/5_related_work_limitations.tex}
\input{section/iclr/6_conclusion.tex}

\input{section/iclr/7_statements.tex}

\bibliography{hongji-version}
\bibliographystyle{iclr2027_conference}

\appendix

\input{section/iclr/appendix_a_measure_theory.tex}
\input{section/iclr/appendix_b_experimental_details.tex}

\input{section/iclr/appendix_c_answer_letter_steering.tex}

\input{section/iclr/appendix_e_basin_mass_geometry.tex}

\clearpage

\input{section/iclr/appendix_f_gsm8k_alignment.tex}
\input{section/iclr/appendix_g_safetybench_probing.tex}
\input{section/iclr/appendix_d_related_work.tex}

\end{document}

%% file: math_commands.tex
\usepackage{amsmath,amsfonts,bm}

\def\eqref#1{equation~\ref{#1}}

\def\1{\bm{1}}

\DeclareMathAlphabet{\mathsfit}{\encodingdefault}{\sfdefault}{m}{sl}
\SetMathAlphabet{\mathsfit}{bold}{\encodingdefault}{\sfdefault}{bx}{n}



%% file: section/iclr/0_abstract.tex
\begin{abstract}
Linear probing and activation steering use linear directions to predict and
control behavior-level concepts such as correctness, safety, and social bias
in question answering. We call these \emph{behavior-level concept directions}.
Drawing on the Linear Representation Hypothesis (LRH) and circuit studies,
these directions are often interpreted as concept representations. Yet
controlled evidence for linear concept structure and local mechanisms does
not establish that these directions represent the intended behavioral
concepts, leaving probing and steering without a unified theoretical account.
We propose the \emph{Answer-Basin Representation Hypothesis} (ABRH): the
model's own answer measure organizes the linear structure of these
directions. For each question, all continuations yielding the same answer
form an answer basin, whose mass is their total probability; these masses
define the model's answer distribution. ABRH posits that its concentration
before generation and the relative mass of each answer after generation are
represented along linear directions shared across questions. Experiments
span four Qwen2.5 and Gemma-3 models on correctness, social bias, and safety
tasks. Decoupling concept labels from basin mass shows that probing and
steering exhibit concept-consistent effects when labels align with mass
orderings, weaken as mass gaps shrink, and reverse under conflict. Probes
trained solely on mass orderings among wrong answers still select correct
answers; a fixed steering vector can instead favor wrong answers when labels
conflict with mass orderings. These results support an answer-measure account
of behavior-level concept directions and of when probing and steering
succeed, fail, or reverse.
\end{abstract}

%% file: section/iclr/1_intro.tex
\section{Introduction}
\vspace{-0.1in}
\label{sec:introduction}

\afterpage{%
\begin{figure*}[!t]
    \centering
    \includegraphics[width=\textwidth]{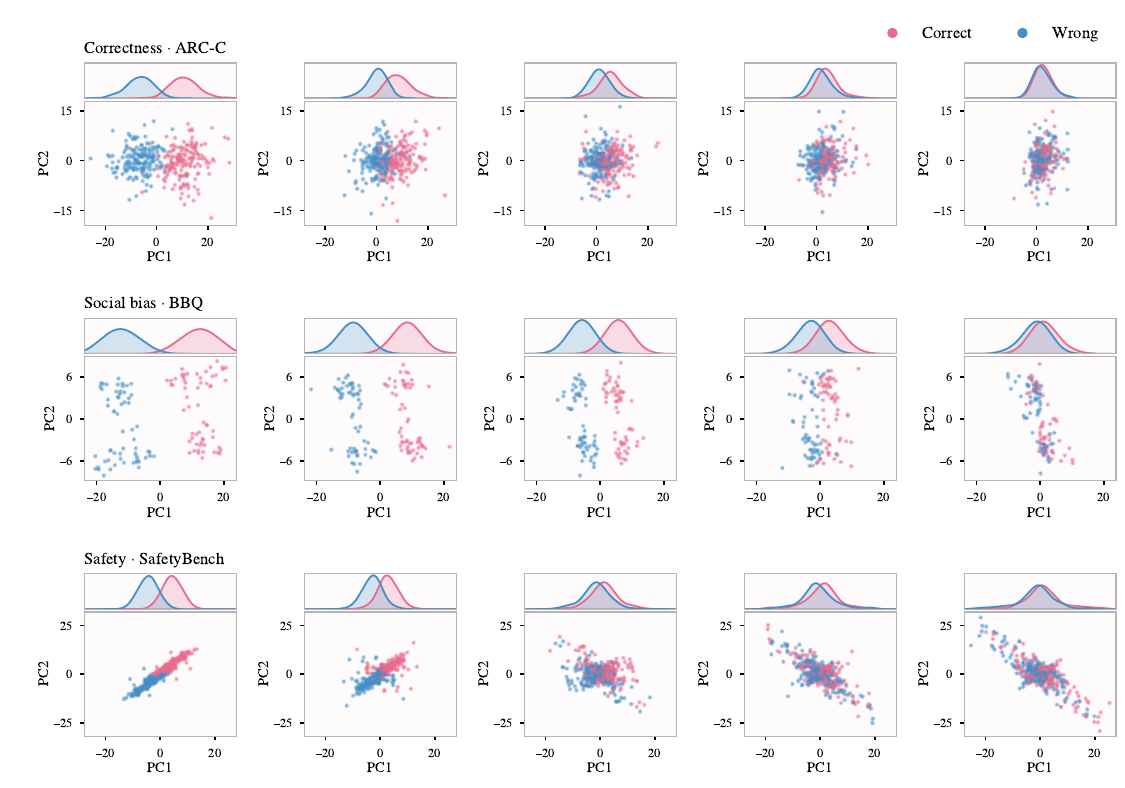}
    \caption{\textbf{Larger answer-basin mass gaps accompany clearer separation of correct and wrong answers.}
    Qwen2.5-7B representations on three attributes: correctness (ARC-C),
    social bias (BBQ), and safety (SafetyBench).
    Columns show correct--wrong answer pairs with decreasing mass gaps
    under the candidate-answer measure. Points are answer representations in a shared PCA
    space within each row; curves show their PC1 distributions.}
    \label{fig:basin-gap-separability}
\end{figure*}%
}

Linear probing and activation steering use linear directions to predict and
control behavior-level concepts such as correctness, safety, and social bias
in question answering. Linear probes predict whether answers are correct,
safe, or biased from hidden states
\citep{azaria2023internal,burns2023discovering,marks2024geometry,cencerrado2025noanswer};
activation steering modifies activations along a direction to change the
model's answer behavior
\citep{li2023inference,zou2023representation,rimsky2024caa,arditi2024refusal}.
We call linear directions targeting these concepts
\emph{behavior-level concept directions}.

These directions are often interpreted as representations of the
corresponding concepts, drawing on the Linear Representation Hypothesis
(LRH) and circuit studies. Representative LRH work establishes connections
between linear concept representations, probing, and steering through
counterfactual formalizations \citep{park2024linear}.
Anthropic's circuit studies trace and validate computational mechanisms in
specific contexts through case studies \citep{lindsey2025biology}.
Behavior-level probing and steering in question answering follow a different
research route. Methods exemplified by Representation Engineering (RepE)
\citep{zou2023representation} take a top-down approach: researchers first
define concepts such as correctness and safety, then recover directions
from label supervision or contrastive data and use them for prediction and
intervention across questions. Yet controlled evidence for linear concept
structure and local mechanisms does not establish that these directions
represent the intended behavioral concepts. Probing and steering still
lack a unified theoretical account of what organizes these directions'
linear structure and why probing and steering succeed, fail, or reverse.

To understand this organization, we start from the generative distribution
of a language model. An autoregressive model predicts tokens sequentially,
and the conditional probabilities at successive steps determine the
probability of a complete continuation. Different continuations can yield
the same answer, as when different solution paths reach the same result in a
math problem \citep{wang2023selfconsistency}. For a given question, we define
an answer's \emph{answer basin} as the set of all continuations yielding it,
and its \emph{basin mass} as their total probability under the model. These
basin masses define the pushforward probability measure over answers
\citep{kallenberg2021foundations}. We call this the \emph{answer probability
measure}, or \emph{answer measure}.

Our central observation is that the linear structure of behavior-level
concept directions emerges with differences in the model's own answer
measure. These directions' concept-level effects in probing and steering
depend on the relationship between concept labels and the answer measure.
Figure~\ref{fig:basin-gap-separability} connects this measure to
representational geometry: across correctness, social bias, and safety,
correct--wrong answer pairs separate more clearly as their basin-mass gap
grows; as the gap shrinks, the two distributions increasingly overlap.

We propose the \textit{Answer-Basin Representation Hypothesis} (ABRH):
\textbf{the model's own answer measure organizes the linear structure of
these directions, with its statistics represented along linear directions
shared across questions.}
Before generation, the hidden state at the question boundary (the final
prompt token) reflects the concentration of the answer distribution. After
generation, the state at the answer boundary (the final answer token)
reflects the relative basin mass of that answer. ABRH brings concentration
before generation and answer-mass ordering after generation into a common
representation framework and relates concept-consistent behavior to
label--measure relations: alignment between labels and mass orderings
yields the intended concept-level effects, smaller mass gaps weaken
separability, and conflict can reverse concept-level effects.

We test this account by decoupling concept labels from answer mass. Our main
experiments span four Qwen2.5 and Gemma-3 models on correctness, social bias,
and safety tasks. Probing experiments show that shared linear directions
read out answer concentration before generation and answer-mass orderings
after generation on held-out questions. Keeping correctness labels and the
training objective unchanged, changing the label--mass relationship in the
training data changes answer selection on the same test set. On ARC-C, BBQ,
and SafetyBench, probes trained under alignment outperform the mean
log-probability baseline, whereas those trained under conflict fall below
it. Probes trained solely on mass orderings among wrong answers still
select correct answers and outperform this baseline on ARC-C and
SafetyBench. These results connect correctness predictions to answer-mass
relations in the probe's training data.

For activation steering, we posit an additional, testable transport
hypothesis: positive intervention along an answer-mass-oriented direction
gives initially higher-mass answers greater relative gains. In our
Gemma-3-4B-PT experiments on BBQ and SafetyBench, a fixed vector constructed
to favor correct answers favors the correct answer when it has greater
initial mass, but favors the wrong answer when it has greater initial mass.
The vector and label orientation remain unchanged, while the label--mass
relationship in the target questions produces opposite concept-level
effects. This result connects a fixed vector's steering behavior to the
target question's answer measure.

The geometry of steering vectors also varies with answer-mass relations in
the source data. Keeping the correct-minus-wrong contrast unchanged,
reversing the source label--mass relationship produces vectors with nearly
opposite orientations in a shared PCA projection. Larger source mass gaps
correspond to longer, more directionally coherent representation contrasts
and larger aggregate vectors; matching aggregate-vector norms reduces much
of the difference in steering strength. Source mass relations thus shape
vector orientation, magnitude, and coherence, while target mass relations
predict which answer a fixed intervention favors. The answer measure links
vector geometry with prediction and intervention in a common account of
concept-consistent effects, their weakening, and their reversals.

Our main contributions are:
\begin{itemize}[nosep, leftmargin=*]
    \item We identify the theoretical gap for behavior-level concept
    directions: controlled evidence for linear concept structure and
    specific mechanisms has not established a unified account of these
    directions and the conditions under which they work in question answering.
    \item We formalize answer basins and the answer measure, and propose
    ABRH to connect concentration representations before generation with
    relative answer-mass representations after generation, together with
    a corresponding intervention hypothesis.
    \item Through experiments that decouple concept labels from answer
    mass, we connect representational geometry, probe predictions, and
    steering effects, showing how concept-consistent effects, their
    weakening, and their reversals vary with label--measure relations.
\end{itemize}

\vspace{-0.1in}

%% file: section/iclr/2_hypothesis.tex
\section{The \textnormal{\textit{Answer-Basin Representation Hypothesis}}}
\vspace{-0.1in}
\label{sec:abrh}

\label{sec:linear-direction-identification}
Under a fixed evaluation protocol, model outputs induce a probability
measure over answer identities. The \textit{Answer-Basin Representation
Hypothesis} connects this measure's statistics to hidden
representations, probing, and steering to explain the
geometry in Figure~\ref{fig:basin-gap-separability} and
Appendix~\ref{app:basin-gap-geometry}.

\vspace{-0.1in}
\subsection{Answer Measures and Answer Basins}
\vspace{-0.1in}
\label{sec:answer-basins}

Throughout, an evaluation instance fixes a question $q$, a countable answer
space $\mathcal A_q$, and an answer-measure protocol. The protocol maps model
outputs to a probability measure $\mu_q$ on $\mathcal A_q$; we suppress the
protocol in the notation and write $m_q(a)=\mu_q(\{a\})$ for answer mass. We
construct this measure using either a generation protocol or a finite-choice
protocol.

Under a generation protocol, let $P_q=P_\theta(\cdot\mid q)$ be the
continuation measure on $\mathcal Y_q$, and let the measurable map
$\phi_q:\mathcal Y_q\to\mathcal A_q$ assign each continuation its answer
identity. The fiber $\mathcal B_q(a)=\phi_q^{-1}(\{a\})$, the answer basin of
$a$, contains all continuations producing it. The pushforward answer measure
and basin masses are
\begin{equation}
    \mu_q=(\phi_q)_\#P_q,
    \qquad
    m_q(a)=\mu_q(\{a\})=P_q\!\left(\mathcal B_q(a)\right).
    \label{eq:answer-basin-distribution}
\end{equation}

The map $\phi_q$ therefore converts probability over token trajectories into
probability over answer identities. Basin mass aggregates all trajectories in a
fiber and defines the model's answer-level probability landscape
\citep{wang2023selfconsistency,farquhar2024semanticentropy}.

A finite-choice protocol fixes $\mathcal A_q=\{a_1,\ldots,a_n\}$ and
uniquely encodes each complete answer with a single option-label token.
Its answer map sends the emitted label to that answer, grouping all
continuations making the same commitment into one basin. Normalizing over
valid labels yields a pushforward measure on the listed answers. We vary
label assignments and align scores by answer identity.
Section~\ref{sec:measuring-answer-basins} establishes logit--basin ordering
and defines presentation aggregation. Mass, concentration, and signed
contrasts apply to either construction.

For $A,A'\stackrel{\mathrm{iid}}{\sim}\mu_q$, the collision concentration
\citep{simpson1949diversity} summarizes how the answer measure is concentrated
and equals the expected mass of an answer drawn from it:
\begin{equation}
    C_2(\mu_q)
    :=\Pr(A=A'\mid q)
    =\sum_{a\in\mathcal A_q}m_q(a)^2
    =\mathbb E_{A\sim\mu_q}[m_q(A)].
    \label{eq:answer-measure-collision}
\end{equation}

For two positive-mass answers, define the log-mass contrast
\begin{equation}
    \Delta_q^{\log m}(a,b)
    :=
    \log\frac{m_q(a)}{m_q(b)}.
    \label{eq:log-mass-contrast}
\end{equation}

\vspace{-0.1in}
\subsection{The \textnormal{\textit{Answer-Basin Representation Hypothesis}}}
\vspace{-0.1in}
\label{sec:representation-hypothesis}

The preceding section defines the answer measure and its two constructions.
We examine hidden states before answer generation at the final prompt token
and after answer generation at the final answer token specified by the
protocol. These positions are the question and answer boundaries,
respectively. Let $h_Q(q)$ denote the question-boundary state and $H_A(q,a)$
the answer-state random variable for answer $a$ under the fixed protocol.
Its generation and finite-choice constructions are given in
Appendix~\ref{app:basin-conditioned-hidden-states}.

\begin{definition}[Answer-conditioned representation]
For each positive-mass answer, define
\begin{equation}
    \bar h_q(a)
    :=
    \mathbb E\!\left[H_A(q,a)\right].
    \label{eq:basin-conditioned-representation}
\end{equation}
\end{definition}

We now ask whether shared linear directions reveal answer mass in these means
and concentration in question-boundary states.

\begin{abrhstatement}
Fix a model, a residual-stream layer, an answer-measure construction, and a
state-extraction protocol. Hidden-state differences are consistently oriented
with signed changes in statistics of the same answer measure at the question
and answer boundaries. There exist directions $w_Q,w_A\in\mathbb R^d$,
shared across questions, such that:

\textbf{At the question boundary.} For any questions $q,q'$ with
$C_2(\mu_q)\neq C_2(\mu_{q'})$,
\begin{equation}
    \bigl[C_2(\mu_q)-C_2(\mu_{q'})\bigr]
    w_Q^\top\bigl[h_Q(q)-h_Q(q')\bigr]
    >0,
    \label{eq:question-final-abrh}
\end{equation}
\textbf{At the answer boundary.} For any question $q$ and positive-mass
answers $a,b$ with $m_q(a)\neq m_q(b)$,
\begin{equation}
    \Delta_q^{\log m}(a,b)
    w_A^\top\bigl[\bar h_q(a)-\bar h_q(b)\bigr]
    >0.
    \label{eq:answer-final-abrh}
\end{equation}
\end{abrhstatement}

In other words, increasing concentration corresponds to positive displacement along
$w_Q$, while increasing answer mass corresponds to positive displacement along
$w_A$.

The \textit{Answer-Basin Representation Hypothesis} describes an idealized
linear ordering of hidden representations by answer-measure statistics.
We learn shared directions from finite data to test how
closely models realize this structure and how well these directions
generalize across questions.\looseness=-1

\paragraph{Answer-measure transport under intervention.}
\label{sec:basin-transport}
The readout hypothesis motivates a corresponding prediction for intervention.
An answer-mass-oriented direction $v$ averages within-question differences of
the answer-conditioned means, each pointing toward the higher-mass answer
(Appendix~\ref{app:basin-direction-construction}). At layer $\ell$, we apply
$h_\ell^{(\alpha)}=h_\ell+\alpha v$ and reapply the same answer-measure
protocol to the intervened model outputs. Denote the resulting measure by
$\mu_{q,\alpha}^{(v)}$. Under generation, this is the pushforward of the
intervened continuation measure; under finite-choice scoring, it is obtained
by applying the same scoring-and-aggregation rule to the intervened candidate
scores. For answers with positive baseline and intervened masses, define
\begin{equation}
    m_{q,\alpha}^{(v)}(a):=\mu_{q,\alpha}^{(v)}(\{a\}),\quad
    G_{q,v}^{(\alpha)}(a):=\log\frac{m_{q,\alpha}^{(v)}(a)}{m_q(a)}.
    \label{eq:intervened-basin-mass}
\end{equation}
For such a direction, we posit an additional transport hypothesis: positive
steering gives initially higher-mass answers greater relative amplification.
For answers $a,b$ with positive baseline and intervened masses and
$m_q(a)\neq m_q(b)$,
\begin{equation}
    \boxed{
    \Delta_q^{\log m}(a,b)
    \left[
        G_{q,v}^{(\alpha)}(a)
        -
        G_{q,v}^{(\alpha)}(b)
    \right]
    >0,
    \qquad
    \alpha>0.}
    \label{eq:basin-amplification-ordering}
\end{equation}
Positive steering is predicted to increase relative preference for the
initially higher-mass answer.\looseness=-1

\paragraph{Concept-level orientation.}
\label{sec:concept-basin-alignment}
An external attribute assigns each answer a label
$C_q:\mathcal A_q\to\{0,1\}$, such as correctness or truthfulness. For a pair
$(a,b)$, labels and mass ordering align when
$[C_q(a)-C_q(b)]\Delta_q^{\log m}(a,b)>0$, with greater mass on the
attribute-positive answer. A negative product indicates conflict, with
greater mass on the attribute-negative answer. Answers with the same label
can still have different basin masses. On binary questions, positive steering
therefore favors the desired answer under alignment and its alternative under
conflict.

For generation protocols, the measure-change identities are developed in
Appendix~\ref{app:answer-change-of-measure}.

\vspace{-0.1in}
\subsection{Instantiating Answer Measures}
\vspace{-0.1in}
\label{sec:measuring-answer-basins}

We instantiate both answer-measure constructions: repeated sampling estimates
the measure induced by generation, while option-label scores read out basin
masses on the listed finite answer space.

\paragraph{Sampling-based estimation.}
Under a fixed generation protocol, draw $K\ge2$ independent continuations and let
$N_q(a)$ count those assigned to answer $a$ by $\phi_q$. We estimate mass and
concentration by
\begin{equation}
    \widehat m_q(a)
    :=
    \frac{N_q(a)}{K},
    \qquad
    \widehat C_{2,q}
    :=
    \frac{\sum_a N_q(a)\bigl(N_q(a)-1\bigr)}{K(K-1)}.
    \label{eq:empirical-question-concentration}
\end{equation}
The second quantity is the fraction of agreeing sample pairs and is unbiased
for $C_2(\mu_q)$.

\paragraph{Candidate-answer measures.}
For $\mathcal A_q=\{a_1,\ldots,a_n\}$, let $\lambda_\pi(i)$ encode
complete answer $a_i$ in presentation $\pi$. Write $z^\pi_{q,i}$ for its
label logit and $m^\pi_q(a_i)$ for its answer-basin mass, normalized over
valid option labels.

\begin{proposition}[Option-label logits and complete-answer basin ordering]
\label{prop:label-basin-ordering}
With one-to-one single-token answer labels, logits and basin masses satisfy
\begin{equation}
    m^\pi_q(a_i)
    =\frac{\exp z^\pi_{q,i}}{\sum_{j=1}^n\exp z^\pi_{q,j}},
    \qquad
    z^\pi_{q,i}-z^\pi_{q,j}
    =\log\frac{m^\pi_q(a_i)}{m^\pi_q(a_j)}.
    \label{eq:label-basin-logit-equivalence}
\end{equation}
Hence $z^\pi_{q,i}>z^\pi_{q,j}$ if and only if
$m^\pi_q(a_i)>m^\pi_q(a_j)$, with ties also preserved.
\end{proposition}

Label probability sums over all continuations making the encoded answer
commitment. A single-token readout thus measures a complete-answer basin
(proof in Appendix~\ref{app:balanced-label-measurement}).

We aggregate the answer-aligned label logits over option presentations:
\begin{equation}
    s_{q,i}:=\frac{1}{|\Pi|}\sum_{\pi\in\Pi}z^\pi_{q,i},
    \qquad
    \widetilde m_q(a_i)
    :=\frac{\exp s_{q,i}}{\sum_{j=1}^n\exp s_{q,j}}.
    \label{eq:balanced-candidate-scoring}
\end{equation}
These masses define the candidate answer measure and its concentration:
\begin{equation}
    \widetilde\mu_q
    :=\sum_{a\in\mathcal A_q}\widetilde m_q(a)\,\delta_a,
    \qquad
    \widetilde C_{2,q}
    :=C_2(\widetilde\mu_q)
    =\sum_{a\in\mathcal A_q}\widetilde m_q(a)^2.
    \label{eq:candidate-answer-measure}
\end{equation}
Here, $\delta_a$ is a unit point mass. Balanced orders assign each answer
to each position equally often. Under a separable position-effect model,
mean logits preserve a shared basin-mass ordering on $\mathcal A_q$.
Appendix~\ref{app:balanced-label-measurement} gives the cancellation proof,
residual bound, and binary probability-averaging result. Scoring protocols
are specified in Appendix~\ref{app:experimental-details}. Hats denote
sampling estimates; tildes denote candidate-scoring quantities.

\FloatBarrier
\vspace{-0.1in}

%% file: section/iclr/3_probing.tex
\section{Probing Answer-Basin Representations}
\vspace{-0.1in}
\label{sec:empirical-foundations}

The \textit{Answer-Basin Representation Hypothesis} predicts that shared
linear directions reflect concentration ordering across questions before
answer generation and basin-mass ordering within a question after answer
generation. We test these predictions on held-out questions using the sampling
estimates and candidate measures in Section~\ref{sec:measuring-answer-basins}.
For readouts after answer generation, we examine performance across
datasets, model families, and parameter scales, its relation to training-pair
mass gaps, and whether readout persists when the paired training answers
share the same concept label.\looseness=-1

\label{sec:probing-answer-basin-contrasts}

For $w=w_A$, Eq.~\ref{eq:answer-final-abrh} gives higher mean probe scores to
higher-mass answers within each question; the linearity derivation is given
in Appendix~\ref{app:basin-conditioned-hidden-states}.

\vspace{-0.1in}
\subsection{Experimental Setup and Metrics}
\vspace{-0.1in}
Before answer generation, probes predict concentration from the final prompt-token state
on GSM8K \citep{cobbe2021gsm8k}, BBQ \citep{parrish2022bbq},
and SafetyBench \citep{zhang2024safetybench}.
Each probe uses 500 training questions; targets are
Eq.~\ref{eq:empirical-question-concentration} for repeated generations or
$\widetilde C_{2,q}$ for candidate measures derived from option scores.
We report pairwise ordering accuracy (Order Acc.) and Spearman correlation
($\rho$); see Appendix~\ref{app:question-final-protocol} for details.

After answer generation, probes use the hidden state at the last token of each candidate answer
appended to the question and learn basin-mass ordering on ARC-Challenge
\citep{clark2018arc} and SafetyBench.
Permutation-aggregated option preferences define $\widetilde\mu_q$.
Low, Medium, and High vary the
training-pair mass gap,
$\delta_q(a,b):=\left|\widetilde m_q(a)-\widetilde m_q(b)\right|$,
to test its relation to readout generalization. Wrong-only uses pairs of
answers that are both incorrect according to the dataset's gold answer but
have different basin masses, using their ordering as supervision.
Basin ordering averages probe--mass ranking agreement over all candidate
pairs within each test question, then across questions. Top-1 correctness
evaluates the highest-scoring answer against the gold label. Baseline definitions, training budgets, and evaluation splits are given in
Appendix~\ref{app:table2-mean-logp}.

\vspace{-0.1in}
\subsection{Before Answer Generation: Answer Concentration}
\vspace{-0.1in}
\label{sec:question-final-basin-concentration}

\begin{wrapfigure}[13]{r}{0.36\textwidth}
    \vspace{-9pt}
    \centering
    \includegraphics[width=0.80\linewidth,trim=8bp 0 0 0,clip]{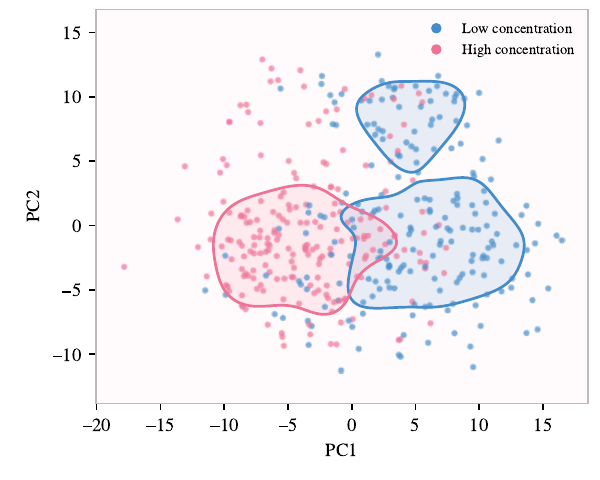}
    \setlength{\abovecaptionskip}{0pt}
    \caption{\footnotesize
PCA projection of the question-final hidden states for 200 high- and
200 low-concentration GSM8K questions (Qwen2.5-7B).
Contours enclose 50\% density regions.}
    \label{fig:question-end-concentration}
\end{wrapfigure}
\begingroup
\emergencystretch=1em
\hyphenpenalty=200
\brokenpenalty=10000

Eq.~\ref{eq:answer-measure-collision} identifies the concentration of the
future answer distribution with the expected basin mass of the answer
eventually realized. We therefore test whether a linear readout from the
hidden state before answer generation preserves concentration ordering across questions,
as posited in Eq.~\ref{eq:question-final-abrh}.

\begin{table}[!t]
    \centering
    \caption{\textbf{Probing answer concentration.}
    Linear scores before answer generation track future answer
    concentration on held-out questions.}
    \label{tab:question-final-basin-concentration}
    \small
    \setlength{\tabcolsep}{3.0pt}
    \renewcommand{\arraystretch}{1.10}
    \begin{tabular*}{\linewidth}{@{\extracolsep{\fill}}lcccccc@{}}
        \toprule
        \multicolumn{1}{c}{\multirow{2}{*}{Model}}
        & \multicolumn{2}{c}{GSM8K}
        & \multicolumn{2}{c}{BBQ}
        & \multicolumn{2}{c}{SafetyBench} \\
        \cmidrule(lr){2-3}\cmidrule(lr){4-5}\cmidrule(l){6-7}
        & Order Acc. $\uparrow$ & $\rho\uparrow$
        & Order Acc. $\uparrow$ & $\rho\uparrow$
        & Order Acc. $\uparrow$ & $\rho\uparrow$ \\
        \midrule
        Qwen2.5-7B base
        & $74.84$ & $0.676$
        & $70.52\pm1.40$ & $0.582$
        & $67.95\pm1.64$ & $0.515$ \\
        Qwen2.5-14B base
        & $77.60$ & $0.720$
        & $73.92\pm1.86$ & $0.657$
        & $67.95\pm0.86$ & $0.521$ \\
        Gemma-3-4B-PT
        & $75.45$ & $0.700$
        & $63.48\pm3.92$ & $0.384$
        & $70.42\pm1.18$ & $0.571$ \\
        Gemma-3-12B-PT
        & $78.20$ & $0.745$
        & $74.48\pm0.86$ & $0.677$
        & $63.42\pm0.82$ & $0.387$ \\
        \bottomrule
    \end{tabular*}
\end{table}

\textbf{Concentration is readable before answer generation.}\\
Table~\ref{tab:question-final-basin-concentration} evaluates concentration
readout on held-out questions. Order Acc. measures how often probe scores
correctly order two test questions by estimated answer concentration.
It exceeds chance across all evaluated models and tasks.
The moderate to strong positive Spearman correlations ($\rho$) show
that higher probe scores tend to correspond to higher estimated
concentration across the test set. Both metrics support the prediction
in Eq.~\ref{eq:question-final-abrh} that a shared linear readout captures
concentration ordering before answer generation.

\WFclear
\interlinepenalty=10000
Figure~\ref{fig:question-end-concentration} visualizes question-final
representations from Qwen2.5-7B for 200 high- and 200 low-concentration
GSM8K questions, grouped by estimated answer concentration $\widehat C_2$.
Contours enclose 50\% density regions. The two groups have different
distributions in the PCA projection.
\par\endgroup

\subsection{After Answer Generation: Answer-Basin Mass}
\label{sec:empirical-setting}

\textbf{Alignment and conflict reverse concept predictions.}
We examine how the relationship between correctness labels and basin mass
in the training data shapes concept predictions on held-out questions.
We train correctness probes on answer pairs with the largest positive
(Align) or most negative (Conflict) signed mass gaps,
$\widetilde m_q(a_{\mathrm{correct}})-\widetilde m_q(a_{\mathrm{wrong}})$,
keeping correct answers as the positive label in both conditions.
On the same held-out questions, Align exceeds the mean-logp baseline
and Conflict falls below it in all twelve model--dataset combinations
on ARC-C, BBQ, and SafetyBench
(Figure~\ref{fig:align-conflict-top1}). The relationship between basin mass
and correctness in training determines the probe's concept orientation.
Additional GSM8K results are reported in
Appendix~\ref{app:gsm8k-align-conflict}.

\begin{figure}[!htbp]
    \centering
    \includegraphics[width=\linewidth]{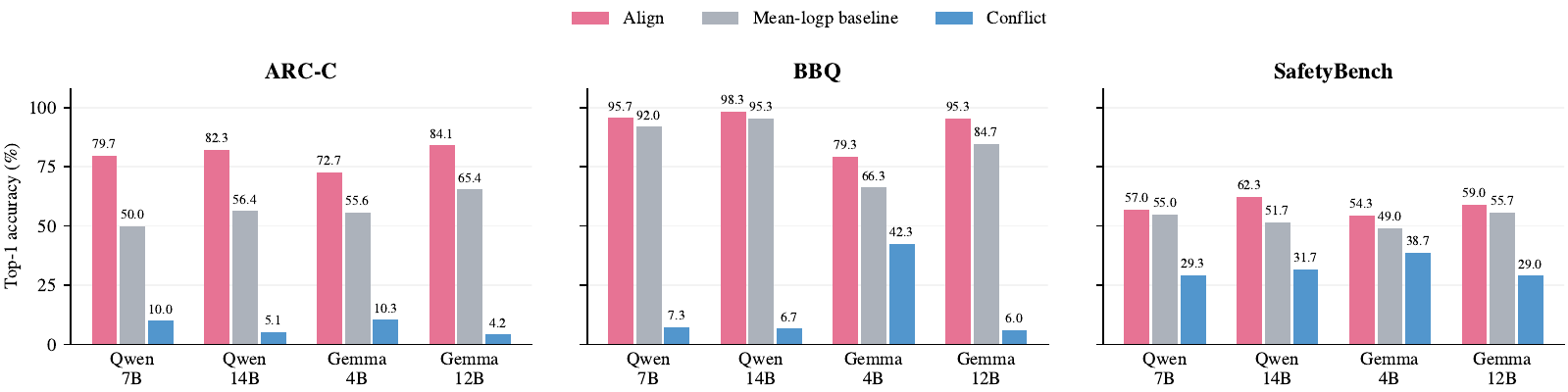}
    \caption{\textbf{Probing under label--basin alignment and conflict.}
    With the same correctness supervision, answer-selection accuracy is
    above the mean-logp baseline under aligned training and below it under
    conflict training.}
    \label{fig:align-conflict-top1}
\end{figure}

Eq.~\ref{eq:answer-final-abrh} predicts that mean answer representations
admit a shared linear ordering by basin mass. We next examine their
geometry on GSM8K using Qwen2.5-7B, averaging the final hidden states of
responses that yield the same answer.
Figure~\ref{fig:gsm8k-answer-geometry} compares correct and wrong answer
groups across decreasing estimated basin-mass gaps. Each point is an
answer group's mean representation; $g$ denotes the median estimated
log mass ratio in each group of questions, and contours enclose 50\%
density regions. Larger-gap groups show clearer separation in the PCA
projection; as the gap narrows, the distributions increasingly overlap.

\begin{figure}[!htbp]
    \centering
    \includegraphics[width=\linewidth]{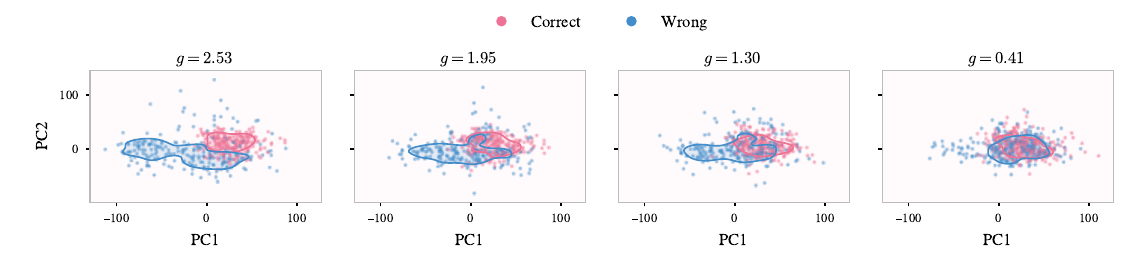}
    \caption{\textbf{Answer-group geometry across basin-mass gaps.}
    Qwen2.5-7B on GSM8K: each point is an answer group's mean representation;
    $g$ is the median estimated log mass ratio in each group of questions.
    Contours enclose 50\% density regions.}
    \label{fig:gsm8k-answer-geometry}
\end{figure}

We then test basin-mass ordering directly by training linear probes on
answer-pair mass comparisons. Table~\ref{tab:basin-order-correctness}
evaluates their ordering and answer-selection accuracy on ARC-Challenge
across model families and training-pair gaps. SafetyBench results are
reported in Appendix~\ref{app:safetybench-basin-probing},
Table~\ref{tab:safetybench-basin-order-correctness}.

\begin{table}[!htbp]
    \centering
    \caption{\textbf{Probing basin mass across training conditions.}
    Larger training-pair mass gaps generally support more accurate basin
    readout and answer selection. Probes trained only on incorrect answers
    still recover basin-mass ordering on held-out questions.}
    \label{tab:basin-order-correctness}
    \footnotesize
    \setlength{\tabcolsep}{3pt}
    \renewcommand{\arraystretch}{1.0}
    \resizebox{\linewidth}{!}{%
    \begin{tabular}{@{}clccccc@{}}
        \toprule
        \multirow{2}{*}{\textbf{Model}}
        & \multicolumn{1}{c}{\multirow{2}{*}{\textbf{Metric (\%)}}}
        & \multirow{2}{*}{\textbf{Mean logp}}
        & \multicolumn{4}{c}{\textbf{Basin-supervised linear probes}} \\
        \cmidrule(lr){4-7}
        & & & \textbf{Low} & \textbf{Medium} & \textbf{High} & \textbf{Wrong-only} \\
        \midrule
        \rowcolor{tableHeader}
        \multicolumn{7}{@{}l}{\strut\textbf{ARC-Challenge}} \\
        \addlinespace[1pt]
        \cellcolor{qwenSeven!8} & Basin ordering & --- & \cellcolor{qwenSeven!52.7}$69.85\,{\scriptstyle\pm\,2.33}$ & \cellcolor{qwenSeven!71.2}\color{black}$77.77\,{\scriptstyle\pm\,0.56}$ & \cellcolor{qwenSeven!77.9}\color{black}$80.43\,{\scriptstyle\pm\,0.27}$ & \cellcolor{qwenSeven!76.8}\color{black}$80.00\,{\scriptstyle\pm\,0.31}$ \\
        \cellcolor{qwenSeven!8}\multirow{-2}{*}{Qwen2.5-7B base} & Top-1 correctness & \cellcolor{qwenSeven!17.4}\color{black}$50.04$ & \cellcolor{qwenSeven!25.4}$55.50\,{\scriptstyle\pm\,4.81}$ & \cellcolor{qwenSeven!61.5}$73.72\,{\scriptstyle\pm\,1.47}$ & \cellcolor{qwenSeven!75.2}\color{black}$79.36\,{\scriptstyle\pm\,0.63}$ & \cellcolor{qwenSeven!71.8}\color{black}$78.03\,{\scriptstyle\pm\,0.54}$ \\
        \addlinespace[1pt]
        \cellcolor{qwenFourteen!8} & Basin ordering & --- & \cellcolor{qwenFourteen!57.7}$72.10\,{\scriptstyle\pm\,2.05}$ & \cellcolor{qwenFourteen!74.4}\color{black}$79.04\,{\scriptstyle\pm\,0.48}$ & \cellcolor{qwenFourteen!82.4}\color{black}$82.17\,{\scriptstyle\pm\,0.31}$ & \cellcolor{qwenFourteen!81.0}\color{black}$81.62\,{\scriptstyle\pm\,0.29}$ \\
        \cellcolor{qwenFourteen!8}\multirow{-2}{*}{Qwen2.5-14B base} & Top-1 correctness & \cellcolor{qwenFourteen!26.9}\color{black}$56.39$ & \cellcolor{qwenFourteen!32.7}$59.80\,{\scriptstyle\pm\,4.10}$ & \cellcolor{qwenFourteen!69.2}$76.94\,{\scriptstyle\pm\,1.28}$ & \cellcolor{qwenFourteen!82.1}\color{black}$82.06\,{\scriptstyle\pm\,0.72}$ & \cellcolor{qwenFourteen!79.3}\color{black}$80.95\,{\scriptstyle\pm\,0.68}$ \\
        \addlinespace[1pt]
        \cellcolor{gemmaFour!8} & Basin ordering & --- & \cellcolor{gemmaFour!38.6}$62.98\,{\scriptstyle\pm\,3.72}$ & \cellcolor{gemmaFour!78.9}\color{black}$80.80\,{\scriptstyle\pm\,0.11}$ & \cellcolor{gemmaFour!79.9}\color{black}$81.19\,{\scriptstyle\pm\,0.12}$ & \cellcolor{gemmaFour!76.4}\color{black}$79.84\,{\scriptstyle\pm\,0.27}$ \\
        \cellcolor{gemmaFour!8}\multirow{-2}{*}{Gemma-3-4B-PT} & Top-1 correctness & \cellcolor{gemmaFour!25.6}\color{black}$55.62$ & \cellcolor{gemmaFour!9.3}$42.78\,{\scriptstyle\pm\,5.71}$ & \cellcolor{gemmaFour!57.2}\color{black}$71.88\,{\scriptstyle\pm\,0.41}$ & \cellcolor{gemmaFour!59.9}\color{black}$73.03\,{\scriptstyle\pm\,0.44}$ & \cellcolor{gemmaFour!52.4}\color{black}$69.70\,{\scriptstyle\pm\,0.72}$ \\
        \addlinespace[1pt]
        \cellcolor{gemmaTwelve!8} & Basin ordering & --- & \cellcolor{gemmaTwelve!45.4}$66.42\,{\scriptstyle\pm\,3.15}$ & \cellcolor{gemmaTwelve!86.2}\color{black}$83.58\,{\scriptstyle\pm\,0.06}$ & \cellcolor{gemmaTwelve!85.7}\color{black}$83.39\,{\scriptstyle\pm\,0.10}$ & \cellcolor{gemmaTwelve!83.6}\color{black}$82.60\,{\scriptstyle\pm\,0.23}$ \\
        \cellcolor{gemmaTwelve!8}\multirow{-2}{*}{Gemma-3-12B-PT} & Top-1 correctness & \cellcolor{gemmaTwelve!43.4}\color{black}$65.41$ & \cellcolor{gemmaTwelve!15.3}$48.36\,{\scriptstyle\pm\,5.12}$ & \cellcolor{gemmaTwelve!90.0}\color{black}$85.70\,{\scriptstyle\pm\,0.44}$ & \cellcolor{gemmaTwelve!90.0}\color{black}$86.56\,{\scriptstyle\pm\,0.20}$ & \cellcolor{gemmaTwelve!85.4}\color{black}$83.28\,{\scriptstyle\pm\,0.22}$ \\
        \bottomrule
    \end{tabular}%
    }
\end{table}

\textbf{Clearer probability contrasts support readout.} Across the evaluated datasets, model families, and parameter scales,
basin-ordering accuracy exceeds chance, supporting linearly readable
candidate-mass information after answer generation. Larger-gap training conditions generally yield higher basin-ordering
accuracy and Top-1 correctness.
This fits the intuition that a more decisive preference between training
answers provides a clearer probability contrast for learning a shared
direction. The accompanying increase in Top-1 indicates that, on these
tasks, improved basin readout also produces selections more consistent
with gold answers.

\subsection{Probing without Concept Supervision}
\label{sec:probing-without-concept-supervision}

Wrong-only trains probes on pairs of incorrect answers with different
basin masses, using their ordering as supervision. These probes still rank
answers in held-out questions by $\widetilde m_q$
(Table~\ref{tab:basin-order-correctness} and
Appendix~\ref{app:safetybench-basin-probing},
Table~\ref{tab:safetybench-basin-order-correctness}), showing that the signal persists
without a contrast in external concept labels. This transfer supports a
shared scoring rule that assigns higher scores to answers with higher basin mass
within each question and extends this relation to candidates in new
questions.

\vspace{-0.1in}

%% file: section/iclr/4_steering.tex
\vspace{-0.1in}
\section{Steering Answer-Basin Representations}
\vspace{-0.1in}
\label{sec:basin-steering}

Eq.~\ref{eq:basin-amplification-ordering} predicts that positive
basin-oriented steering increases the relative preference for initially
higher-mass answers. We test how this response affects answer selection
when basin ordering aligns or conflicts with concept labels, then examine
how source mass separation shapes the steering direction.

\paragraph{Experimental setup and metrics.}
For each dataset, we construct a label-defined steering vector by averaging
correct-minus-wrong answer-conditioned representation contrasts,
$v=\mathbb{E}_q[\bar h_q(a_T)-\bar h_q(a_O)]$, where
$C_q(a_T)=1$ and $C_q(a_O)=0$. Representations are extracted at the answer
boundary. We hold $v$ fixed across alignment and conflict, defined by the
positive and negative signs of $\Delta_q^{\log m}(a_T,a_O)$, respectively.
We measure the behavioral response as the change in length-normalized
answer log-probability. For source geometry, we relate the signed projection
of $\bar h_q(a_T)-\bar h_q(a_O)$ to the log-mass contrast
$\Delta_q^{\log m}(a_T,a_O)$ and compare individual and aggregate contrast
norms across basin-mass gaps.
Appendix~\ref{app:steering-geometry-protocol} gives the detailed protocols.

\vspace{-0.1in}
\subsection{Why Does Steering Succeed or Fail?}
\vspace{-0.1in}
\label{sec:steering-alignment}

We test Gemma-3-4B-PT on BBQ and SafetyBench. We add $\alpha v$ to the
residual stream after block 19 for BBQ and block 16 for SafetyBench,
at the final token of the \texttt{Answer:} prefix and every subsequent
answer-prediction position.

\begin{figure}[!htbp]
    \centering
    \includegraphics[width=\linewidth]{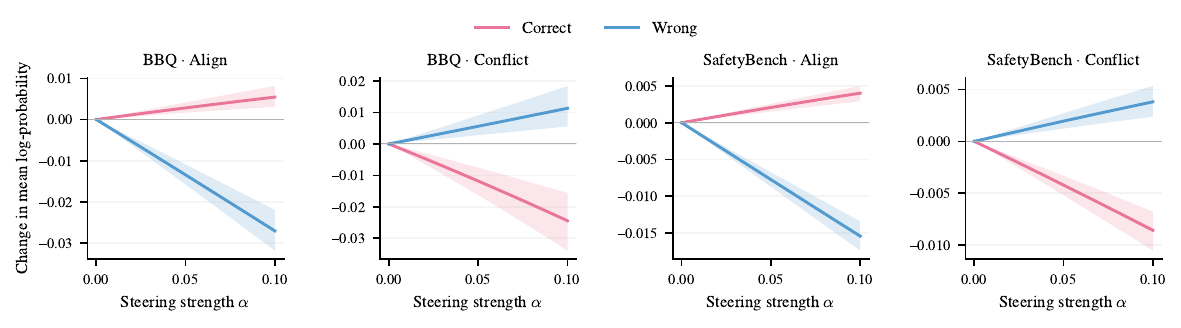}
    \caption{\textbf{A fixed label-defined steering vector produces opposite answer preferences.}
    Changes in length-normalized answer log-probability on Gemma-3-4B-PT under
    alignment and conflict. Shaded bands denote 95\% bootstrap intervals.}
    \label{fig:steering-letter-transport}
\end{figure}

\textbf{Target alignment shapes steering.}
Figure~\ref{fig:steering-letter-transport} shows that the correct answer
receives the larger behavioral response under alignment, whereas
the incorrect answer receives the larger response under conflict.
The relative effect reverses on both datasets despite the fixed vector
and unchanged label orientation. Thus, a direction constructed to favor
correct answers can favor incorrect answers when they carry greater
initial answer mass. This pattern connects the behavioral effect of
label-defined steering to the answer measure.

The result complements the probing dissociation in
Section~\ref{sec:empirical-setting}: source label--basin
alignment shapes learned concept predictions, while target label--basin
alignment shapes the relative effect of a fixed intervention.
Appendix~\ref{app:complete-answer-steering} reports the corresponding answer-letter logit responses
using the same vectors and question pairs.

\vspace{-0.1in}
\subsection{How Does Basin Separation Shape Steering Vectors?}
\vspace{-0.1in}
\label{sec:steering-strength}
\label{sec:basin-geometry}

The basin-dependent reversal in Section~\ref{sec:steering-alignment}
motivates examining how source basin structure shapes the vectors themselves.

\textbf{Source alignment changes projected steering directions.}
Using the Align and Conflict source pools from
Section~\ref{sec:empirical-setting}, we select 80 question pairs per
condition and average their correct-minus-wrong answer-conditioned
representation contrasts. Thus, both vectors retain the same label orientation.
Figure~\ref{fig:steering-vector-pca} shows their joint PCA projections for
Qwen2.5-7B base and Gemma-3-4B-PT on ARC-C, BBQ, and SafetyBench. Switching the sign of the source log-mass contrast
produces opposing projected directions while preserving the correct--wrong
labels. This connects source basin alignment to the geometry of
label-defined steering vectors.

\begin{figure}[!htbp]
    \centering
    \includegraphics[width=\linewidth]{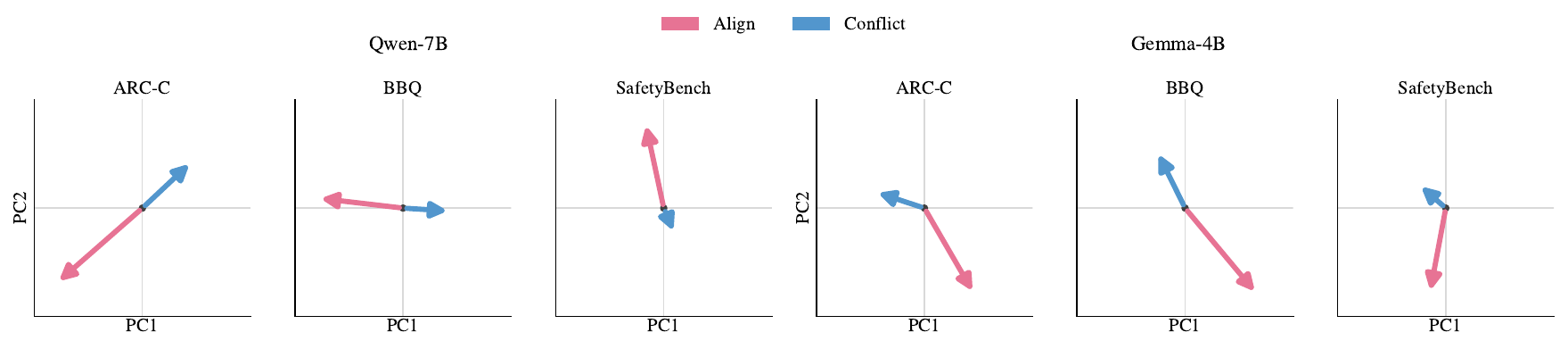}
    \caption{\textbf{Steering-vector projections under source alignment and conflict.}
    Each vector averages 80 correct-minus-wrong answer-conditioned representation contrasts.
    Within each panel, PCA is fitted jointly to the two sets of contrasts;
    arrows show the unnormalized vectors projected onto PC1 and PC2 from
    a common origin.}
    \label{fig:steering-vector-pca}
\end{figure}

We next examine how gap magnitude shapes individual contrasts and their
aggregate. Using Qwen2.5-7B answer-letter states on BBQ and SafetyBench,
we compare signed representation contrasts with answer log-mass contrasts.

\textbf{Larger gaps, more coherent contrasts.} Signed projection grows with the answer log-mass contrast
(Figure~\ref{fig:basin-geometry}). Grouping source questions by absolute
basin-mass gap reveals longer contrasts and more consistent directions at
larger gaps. These contrasts accumulate into larger aggregate vectors,
linking basin separation to both displacement magnitude and directional
coherence. Sample-wise normalization preserves the aggregate direction, so magnitude
weighting has little effect on its orientation.

\begin{figure}[!htbp]
    \centering
    \includegraphics[width=0.95\linewidth]{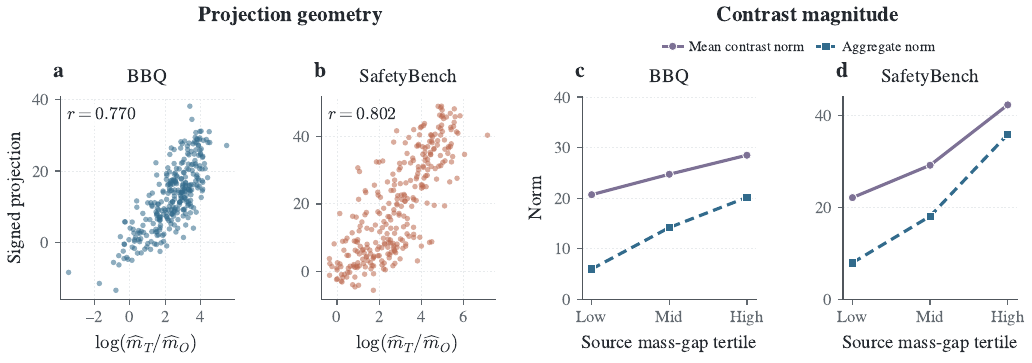}
    \caption{\textbf{Basin separation and steering-vector geometry.}
    Signed representation projections increase with the answer log-mass contrast.
    Larger mass gaps are associated with longer, more directionally coherent
    contrasts that form larger aggregate steering vectors.}
    \label{fig:basin-geometry}
\end{figure}

\textbf{Vector magnitude contributes to steering strength.} Raw averaging gives longer answer-conditioned representation contrasts greater weight.
Vectors constructed from larger-gap sources produce stronger mean
transport effects, and matching aggregate-vector norms reduces much of
this difference. Together with the directional-coherence results, these
observations connect source basin separation to the geometry of the
constructed intervention. Basin structure thus enters the account at
both ends of steering: source separation is reflected in vector magnitude
and coherence, while target alignment predicts which answer receives the
greater relative gain.

%% file: section/iclr/5_related_work_limitations.tex
\section{Related Work and Limitations}
\vspace{-0.1in}
\label{sec:related-work}

\paragraph{Concept Representations and Information Geometry.}
\citet{park2024linear} take concept variables and their linear representations
as a starting point for connecting readout, geometry, and steering.
Further developments are discussed in Appendix~\ref{app:detailed-related-work}.
The \textit{Answer-Basin Representation Hypothesis} takes the answer measure as the organizing object of answer-related
representations and interprets concept-related linear structure through
its statistics. Within this account, the alignment between external labels
and the answer measure determines when statistical directions exhibit
concept-consistent readout and control effects.

\paragraph{Answer Distributions and Semantic Uncertainty.}
Self-consistency aggregates reasoning paths by answer for answer selection
\citep{wang2023selfconsistency}; semantic entropy aggregates generations by
meaning to estimate uncertainty \citep{kuhn2023semantic,farquhar2024semanticentropy}.
Semantic entropy probes establish a predictive relationship between these
statistics and hidden states \citep{kossen2024semanticprobes}.
Grouping outputs by answer or meaning connects these methods to
answer-level probability measures. The \textit{Answer-Basin Representation
Hypothesis} links concentration readout before generation and within-question
basin-mass ordering after generation through a common answer measure.

\paragraph{Concept Probing and Activation Steering.}
Hidden-state probes predict truthfulness and correctness
\citep{azaria2023internal,marks2024geometry,cencerrado2025noanswer}.
Analyses of disagreement between probe predictions and model outputs
examine the interpretation of such readouts \citep{liu2023cognitive}.
Activation steering uses attribute or contrastive directions to
modify behavior \citep{turner2023activation,li2023inference,zou2023representation,rimsky2024caa,arditi2024refusal}.
We examine how the relationship between concept labels and answer mass
shapes the predictions learned by probes and the effects of fixed steering
vectors.

\paragraph{Limitations.}
In open-ended question answering, estimating the
answer measure through free-form sampling faces finite-sample error and
uncertainty in grouping responses by answer. Many of our experiments therefore
use candidate-answer measures defined over fixed candidate-answer spaces.
Extending these results to a broader range of open-ended tasks requires more
accurate methods for answer grouping and answer-measure estimation.

\vspace{-0.1in}

%% file: section/iclr/6_conclusion.tex
\section{Conclusion}
\vspace{-0.1in}
\label{sec:conclusion}

The \textit{Answer-Basin Representation Hypothesis} connects the semantics of answer-level linear directions to the
probability structure of the model's answer basins. It brings linear
probing and activation steering into a common account: hidden states
expose statistics of the answer measure, and interventions reshape that
measure. The complementary probing and steering dissociations support
a relational interpretation of concept semantics, in which a direction's
concept-level predictions and effects depend on how external labels align
with endogenous answer-basin structure. Together, these findings place the model's answer distribution at the center
of understanding what linear directions represent and how they influence behavior.

%% file: section/iclr/7_statements.tex

\subsection*{AI Use Statement}
In this work, we used generative AI tools from OpenAI (ChatGPT/Codex) and
Anthropic (Claude/Claude Code) to help implement experiment code and to
automate the execution of our experiments across multiple random seeds. We
also used these tools for retrieval and discovery, such as finding related
work, and for polishing the English writing to improve readability. All
AI-assisted work underwent author review: we tested the AI-assisted code,
checked the automated runs before reporting their results, checked every
retrieved work against its original source before citing it, and verified
the technical accuracy of the polished text. We take responsibility for the
final content of this work, including text, claims or artifacts produced
with the aid of generative AI.

\subsection*{Ethics Statement}
This work analyzes the hidden representations of publicly released
open-weight language models from the Qwen2.5 and Gemma-3 families on public
benchmarks: GSM8K, ARC-Challenge, BBQ, and SafetyBench. It involves no human
subjects, crowdsourced annotation, or personal data, and we use every model
and dataset under its license and terms of use. BBQ and SafetyBench include
socially biased and unsafe answer options by design; we use them only to
measure model representations and answer preferences. Our findings bear on
the reliable use of interpretability tools for bias and safety. Probes
trained to detect such attributes can reverse their predictions when concept
labels conflict with answer mass, and a steering vector built to favor
correct answers can raise the relative preference for incorrect answers on
BBQ and SafetyBench. The same mechanism could shift an open-weight model
toward undesirable answers. Exploiting it, however, requires white-box access
to activations and uses only standard activation addition, so it adds little
capability beyond existing steering methods.

\subsection*{Reproducibility Statement}
Section~\ref{sec:abrh} defines answer basins, answer measures, and the
statistics that our probes and steering vectors target.
Appendix~\ref{app:abrh-measure-theory} states the underlying
measure-theoretic assumptions, derives the relation between probe scores and
answer-conditioned representations, and proves
Proposition~\ref{prop:answer-level-change-of-measure}.
Section~\ref{sec:measuring-answer-basins} and
Appendix~\ref{app:experimental-details} specify how we estimate answer
measures and run each experiment, including prompts, sampling and
answer-parsing rules, option permutations, extraction layers, intervention
positions and strengths, regularization, and data splits. The same appendix
gives the number of source-set selections and bootstrap resamples behind the
reported standard deviations and confidence intervals.
Appendix~\ref{app:basin-gap-geometry} describes the answer-pair selection and
PCA procedures behind the geometry figures. All models and datasets we use
are publicly available.
\ificlrfinal
Source code is available at
\url{https://github.com/V1centNevwake/answer-basin-representation}.
\else
The supplementary material contains anonymized implementations of the core methods used in our experiments. We will release the complete experimental code publicly upon publication.
\fi

%% file: section/iclr/appendix_a_measure_theory.tex
\vspace{-0.1in}
\section{Measure-Theoretic Details}
\vspace{-0.1in}
\label{app:abrh-measure-theory}

Let $(\mathcal Y_q,\mathcal F_q,P_q)$ be the continuation probability space,
$(\mathcal A_q,\mathcal G_q)$ a countable answer space, and
$\phi_q:(\mathcal Y_q,\mathcal F_q)\to(\mathcal A_q,\mathcal G_q)$ measurable \citep{kallenberg2021foundations}.
The fibers $\mathcal B_q(a)=\phi_q^{-1}(\{a\})$ form a measurable partition of
$\mathcal Y_q$, so $\sum_a m_q(a)=1$. For $\Delta_q=\{(a,a):a\in\mathcal A_q\}$,
Eq.~\ref{eq:answer-measure-collision} follows:
\begin{equation}
    (\mu_q\otimes\mu_q)(\Delta_q)
    =
    \sum_a m_q(a)^2
    =
    \int_{\mathcal A_q}m_q(a)\,\mathrm d\mu_q(a).
\end{equation}

\vspace{-0.1in}
\subsection{Basin-Conditioned Hidden-State Measures}
\vspace{-0.1in}
\label{app:basin-conditioned-hidden-states}

Let $\tau_q(y)$ be the final token position of the answer in continuation $y$,
as specified by the fixed evaluation protocol. The hidden state after answer
generation is defined by
\begin{equation}
    h_A(q,y):=h_{\tau_q(y)}(q,y).
    \label{eq:answer-realization-state}
\end{equation}
The joint hidden-state and answer measure is
\begin{equation}
    \Lambda_q
    :=
    \bigl(h_A(q,\cdot),\phi_q\bigr)_\#P_q.
    \label{eq:hidden-answer-joint-measure}
\end{equation}
For $m_q(a)>0$, its answer-conditioned hidden-state measure is
\begin{equation}
    \nu_q^a(D)
    :=
    \frac{\Lambda_q(D\times\{a\})}{m_q(a)},
    \qquad D\in\mathcal B(\mathbb R^d).
\end{equation}
For generation protocols, $H_A(q,a)$ has distribution $\nu_q^a$, so its
first moment is $\bar h_q(a)=\int h\,\mathrm d\nu_q^a(h)$. For finite-choice
protocols, $H_A(q,a)$ is the state extracted while scoring candidate $a$,
with any averaging over presentation variants specified by the fixed
state-extraction protocol. Both constructions yield
Equation~\ref{eq:basin-conditioned-representation}.

A linear probe scores individual hidden states after answer generation. To
interpret these scores at the level of the answer measure, we average over
the states associated with each answer under the fixed extraction protocol.
Using $H_A(q,a)$ from Section~\ref{sec:representation-hypothesis}, for
$f_w(h)=w^\top h+\beta$ and positive-mass answers $a,b$ whose associated
states have finite first moments, linearity gives
\begin{equation}
    \mathbb E[f_w(H_A(q,a))]
    -
    \mathbb E[f_w(H_A(q,b))]
    =
    w^\top\bigl[\bar h_q(a)-\bar h_q(b)\bigr].
    \label{eq:conditional-probe-contrast}
\end{equation}
For $w=w_A$, Eq.~\ref{eq:answer-final-abrh} gives higher mean scores to
higher-mass answers within each question.

\vspace{-0.1in}
\subsection{Basin-Oriented Direction Construction}
\vspace{-0.1in}
\label{app:basin-direction-construction}

Let $\mathcal S$ be a source distribution over
question--answer triples $(q_s,a_s,b_s)$ with positive, unequal basin masses
and integrable representation differences. The ideal basin-oriented
displacement is
\begin{equation}
    v_{\mathcal S}
    :=
    \mathbb E_{(q_s,a_s,b_s)\sim\mathcal S}
    \!\left[
        \operatorname{sgn}(\Delta_{q_s}^{\log m}(a_s,b_s))
        \bigl(\bar h_{q_s}(a_s)-\bar h_{q_s}(b_s)\bigr)
    \right].
    \label{eq:basin-oriented-steering-direction}
\end{equation}

Eq.~\ref{eq:answer-final-abrh} implies
$w_A^\top v_{\mathcal S}>0$, so the source construction has a positive
component along the shared basin-mass direction. The remaining transformer
layers and autoregressive rollout map the translated state nonlinearly into a
new continuation measure.

\vspace{-0.1in}
\subsection{Answer-Level Change of Measure}
\vspace{-0.1in}
\label{app:answer-change-of-measure}

Assume $P_{q,\alpha}^{(v)}\ll P_q$. Measurability of $\phi_q$ gives
$\mu_{q,\alpha}^{(v)}\ll\mu_q$. Define
\begin{equation}
    r_{q,\alpha}^{(v)}
    :=
    \frac{\mathrm dP_{q,\alpha}^{(v)}}{\mathrm dP_q},
    \qquad
    g_{q,\alpha}^{(v)}
    :=
    \frac{\mathrm d\mu_{q,\alpha}^{(v)}}{\mathrm d\mu_q}.
    \label{eq:trajectory-answer-rn-derivatives}
\end{equation}

\begin{proposition}[Answer-level change of measure]
\label{prop:answer-level-change-of-measure}
For $\mu_q$-almost every answer $a$,
\begin{equation}
    \frac{\mathrm d\mu_{q,\alpha}^{(v)}}{\mathrm d\mu_q}(a)
    =
    \mathbb E_{P_q}\!\left[
        \frac{\mathrm dP_{q,\alpha}^{(v)}}{\mathrm dP_q}(Y)
        \;\middle|\; \phi_q(Y)=a
    \right].
    \label{eq:answer-rn-conditional-expectation}
\end{equation}
\end{proposition}

\begin{proof}
For every $B\in\mathcal G_q$,
\begin{align*}
    \mu_{q,\alpha}^{(v)}(B)
    &=
    P_{q,\alpha}^{(v)}\!\left(\phi_q^{-1}(B)\right)\\
    &=
    \int_{\phi_q^{-1}(B)}
    r_{q,\alpha}^{(v)}(y)\,\mathrm dP_q(y)\\
    &=
    \int_B
    \mathbb E_{P_q}\!\left[
        r_{q,\alpha}^{(v)}(Y)
        \mid\phi_q(Y)=a
    \right]
    \mathrm d\mu_q(a).
\end{align*}
The conditional expectation is the Radon--Nikodym derivative of the
intervened answer measure with respect to the baseline answer measure.
\end{proof}

For every positive-mass atom,
\begin{equation}
    g_{q,\alpha}^{(v)}(a)
    =
    \frac{\mu_{q,\alpha}^{(v)}(\{a\})}{\mu_q(\{a\})},
\end{equation}
whose logarithm, when the intervened mass is also positive, is the basin amplification in
Equation~\ref{eq:intervened-basin-mass}.

\vspace{-0.1in}
\subsection{Answer Odds and Evaluation Attributes}
\vspace{-0.1in}
\label{app:answer-odds-evaluation}

By the definition of $G_{q,v}^{(\alpha)}$, steering updates pairwise answer
log odds as
\begin{equation}
    \log\frac{m_{q,\alpha}^{(v)}(a)}{m_{q,\alpha}^{(v)}(b)}
    =
    \Delta_q^{\log m}(a,b)
    +
    \left[G_{q,v}^{(\alpha)}(a)-G_{q,v}^{(\alpha)}(b)\right].
    \label{eq:steered-basin-log-odds}
\end{equation}
Under the transport hypothesis in Eq.~\ref{eq:basin-amplification-ordering},
positive steering increases the log odds of $a$ over $b$ when
$m_q(a)>m_q(b)$.

For the evaluation attribute $C_q$ defined in
Section~\ref{sec:representation-hypothesis}, its probability and intervention-induced
change are obtained by integrating over the answer measure:
\begin{equation}
    \pi_q(C=1)=\int C_q\,\mathrm d\mu_q,
    \qquad
    \Delta\pi_q^{(v,\alpha)}(C)
    =
    \int C_q\,\mathrm d
    \bigl(\mu_{q,\alpha}^{(v)}-\mu_q\bigr).
    \label{eq:concept-probability-from-basins}
\end{equation}

\vspace{-0.1in}
\subsection{Option Labels, Answer Basins, and Balanced Scoring}
\vspace{-0.1in}
\label{app:balanced-label-measurement}

Fix the finite answer space $\mathcal A_q=\{a_1,\ldots,a_n\}$ and an
option presentation $\pi$ with prompt $x^\pi_q$. Each complete answer $a_i$
has a unique single-token label $\lambda_\pi(i)$. Let
$\mathcal L_q=\{\lambda_\pi(1),\ldots,\lambda_\pi(n)\}$ be the valid
labels and $E_q^\pi=\{Y_1\in\mathcal L_q\}$ the valid-response event at
the answer-readout position. The choice protocol normalizes over these
labels, so its continuation measure is
$P_q^\pi=P_\theta(\cdot\mid x_q^\pi,E_q^\pi)$, supported on the
valid-response space $\mathcal Y_q^\pi$.

On this space, define $\phi_q^\pi(y)=a_i$ when
$y_1=\lambda_\pi(i)$. The answer basins and pushforward measure are
\begin{equation}
    \mathcal B_q^\pi(a_i)
    :=(\phi_q^\pi)^{-1}(\{a_i\})
    =\{y\in\mathcal Y_q^\pi:y_1=\lambda_\pi(i)\},
    \qquad
    \mu_q^\pi:=(\phi_q^\pi)_\#P_q^\pi.
    \label{eq:label-answer-basin-pushforward}
\end{equation}
These basins partition the valid-response space and are indexed by complete
answer identity. Each includes every continuation following its label
commitment; the conditional probability over these continuations sums to one.

\begin{proof}[Proof of Proposition~\ref{prop:label-basin-ordering}]
The basin mass is
\begin{equation}
    m_q^\pi(a_i)
    =P_q^\pi(\mathcal B_q^\pi(a_i))
    =\frac{P_\theta(Y_1=\lambda_\pi(i)\mid x_q^\pi)}
    {\sum_{j=1}^n P_\theta(Y_1=\lambda_\pi(j)\mid x_q^\pi)}
    =\frac{\exp z_{q,i}^\pi}{\sum_{j=1}^n\exp z_{q,j}^\pi}.
    \label{eq:label-event-basin-mass}
\end{equation}
The vocabulary-level softmax normalizer cancels in the second ratio. Taking
logarithms and subtracting gives
$z_{q,i}^\pi-z_{q,j}^\pi=\log[m_q^\pi(a_i)/m_q^\pi(a_j)]$,
which preserves strict order and ties.
\end{proof}

The probability event is the complete answer encoded by the label. Its mass
aggregates all continuations in the answer basin. A single label token
therefore measures a complete-answer basin, regardless of the number of
tokens in the option text.

\paragraph{Balanced option assignments.}
To describe presentation effects within this finite answer space, fix a
positive answer measure $\mu_q$ with masses $m_q(a_i)$ summing to one,
shared across presentations. Assume the separable position-effect model
\begin{equation}
    z^\pi_{q,i}
    =\log m_q(a_i)+c^\pi_q+b_q(k_\pi(i)),
    \label{eq:separable-label-measurement}
\end{equation}
where $k_\pi(i)$ is the position of $a_i$, $c^\pi_q$ is common to all
candidates in that order, and $b_q$ is a position effect independent of
answer identity. These terms may depend on the question and fixed answer
space. An order set $\Pi$ is balanced when every answer occupies every
position equally often.

\begin{proposition}[Basin-mass ordering under balanced label scoring]
\label{prop:balanced-basin-ordering}
Under Eq.~\ref{eq:separable-label-measurement} and balanced $\Pi$, the
mean-logit construction in Eq.~\ref{eq:balanced-candidate-scoring} satisfies
\begin{equation}
    \widetilde m_q(a_i)=m_q(a_i),
    \qquad
    s_{q,i}-s_{q,j}=\log\frac{m_q(a_i)}{m_q(a_j)}.
    \label{eq:candidate-basin-mass-recovery}
\end{equation}
Thus mean label logits preserve the complete-answer basin-mass ordering on
$\mathcal A_q$, including ties.
\end{proposition}

\begin{proof}[Proof of Proposition~\ref{prop:balanced-basin-ordering}]
Each candidate occupies each position $|\Pi|/n$ times. Averaging its mapped
label logits therefore gives
\begin{equation}
    s_{q,i}
    =\log m_q(a_i)
    +\frac{1}{|\Pi|}\sum_{\pi\in\Pi}c^\pi_q
    +\frac{1}{n}\sum_{k=1}^n b_q(k)
    =\log m_q(a_i)+K_q,
    \label{eq:balanced-label-common-offset}
\end{equation}
where $K_q$ is common to all candidates. Pairwise score differences equal
log basin-mass ratios. Exponentiating and normalizing gives
\[
    \widetilde m_q(a_i)
    =\frac{e^{K_q}m_q(a_i)}{e^{K_q}\sum_{j=1}^n m_q(a_j)}
    =m_q(a_i),
\]
which proves Eq.~\ref{eq:candidate-basin-mass-recovery} and preserves both
strict order and ties.
\end{proof}

Averaging label log probabilities yields the same conclusion. Within each
order, log probability subtracts the shared log-partition function from every
candidate logit; this changes only $c^\pi_q$.

\paragraph{Residual distortions.}
Allow answer-dependent deviations from the separable model:
\begin{equation}
    z^\pi_{q,i}
    =\log m_q(a_i)+c^\pi_q+b_q(k_\pi(i))+e^\pi_{q,i},
    \qquad
    \bar e_{q,i}:=\frac{1}{|\Pi|}\sum_{\pi\in\Pi}e^\pi_{q,i}.
    \label{eq:label-measurement-residual}
\end{equation}
The same cancellation yields the exact contrast identity
\begin{equation}
    \log\frac{\widetilde m_q(a_i)}{\widetilde m_q(a_j)}
    =\log\frac{m_q(a_i)}{m_q(a_j)}+\bar e_{q,i}-\bar e_{q,j}.
    \label{eq:residual-basin-contrast}
\end{equation}
If $\max_i|\bar e_{q,i}|\leq\varepsilon_q$, each log-ratio error is at
most $2\varepsilon_q$. Consequently,
$|\log[m_q(a_i)/m_q(a_j)]|>2\varepsilon_q$ guarantees preservation of the
pair's ordering. A unique highest-mass answer remains the highest-scoring
candidate when its log-mass gap to every competitor exceeds this bound.
Smaller gaps and ties have no such guarantee. The residual includes
answer-specific context effects and content--position interactions that
balanced orders need not cancel. The bound is conditional on their
answer-dependent averages; it does not estimate their size in trained models.

\paragraph{Binary probability averaging.}
For the binary protocol that averages within-order normalized probabilities,
let $d=\log[m_q(a_1)/m_q(a_2)]$,
$B=b_q(1)-b_q(2)$, and $\sigma(t)=(1+e^{-t})^{-1}$.
Under the separable model, averaging the two answer-aligned probabilities
from the swapped orders gives
\begin{equation}
    \widetilde m_q^{\mathrm{prob}}(a_1)
    =\tfrac12\bigl[\sigma(d+B)+\sigma(d-B)\bigr],
    \qquad
    \widetilde m_q^{\mathrm{prob}}(a_2)
    =1-\widetilde m_q^{\mathrm{prob}}(a_1).
    \label{eq:binary-probability-ordering}
\end{equation}
The first expression is strictly increasing in $d$ and equals $1/2$ at
$d=0$, since $\sigma(B)+\sigma(-B)=1$. Hence it preserves basin-mass
ordering and ties. It generally differs from $\sigma(d)$, so this protocol
does not recover $m_q(a_i)$ or its concentration exactly.
The residual bound above concerns averaging logits or log probabilities;
it does not automatically extend to averaging normalized probabilities.

%% file: section/iclr/appendix_b_experimental_details.tex
\vspace{-0.1in}
\section{Experimental Details for Answer-Basin Analysis}
\vspace{-0.1in}
\label{app:experimental-details}

\vspace{-0.1in}
\subsection{Concentration Probes Before Answer Generation}
\vspace{-0.1in}
\label{app:question-final-protocol}

For each GSM8K question, we estimate future answer concentration from
$K=64$ independent generations using Eq.~\ref{eq:empirical-question-concentration}.
A linear probe maps the final prompt-token state to $\widehat C_{2,q}$ and
is evaluated on 200 question-disjoint GSM8K questions, with the same
generation template and split across models.

For BBQ and SafetyBench, we construct $\widetilde\mu_q$ by exhaustively
permuting option-to-letter assignments, mapping the letter logits back to
the original options, and applying softmax to their averaged logits.
The target is $\widetilde C_{2,q}=\sum_a\widetilde m_q(a)^2$, the concentration
of this candidate measure. BBQ uses
disambiguated questions
with the two substantive options; SafetyBench uses English questions with
their original two, three, or four options. Each probe uses the state before answer generation from an option-free
prompt with standardized ridge regression and validation-selected
regularization. We evaluate on 5,512 held-out BBQ questions and 2,667 held-out
SafetyBench questions. Table~\ref{tab:question-final-basin-concentration}
reports their means over five training sets and sample standard deviations
for Order Acc.

\vspace{-0.1in}
\subsection{Probes After Answer Generation and Mean-Logp Baseline}
\vspace{-0.1in}
\label{app:table2-mean-logp}

For pairs with
unequal masses, random ranking is correct with probability $1/2$. For a
question with $n_q$ candidates and one correct answer, random Top-1 accuracy
is $1/n_q$. Mean logp separately ranks
individual candidate strings by length-normalized log probability.

Matched Low, Medium, and High conditions share a source-question budget
and a question-disjoint test set.
On SafetyBench, each Low, Medium, and High fit for Qwen2.5-7B
and Qwen2.5-14B uses 100 source questions, with one answer pair per question.
The Gemma Low, Medium and High conditions use 100 source questions.
Wrong-only uses 100 training questions, each contributing its highest- and
lowest-mass incorrect answers under the candidate measure.

For Table~\ref{tab:basin-order-correctness}, each model scores a candidate
independently under \texttt{Question: \{question\}} followed by a newline and
\texttt{Answer: \{candidate\}}. We average the conditional log probabilities
of the candidate continuation tokens, including its leading space. The question
prefix and special tokens are excluded from the average. Tokenization retains
the checkpoint's default BOS behavior, and the full prompt is checked to preserve
the prefix token sequence. Top-1 answer selection uses the highest scalar
score, with no fitted parameters.

Evaluation uses all 1,165 four-choice ARC-Challenge test questions.
The baseline reports one Top-1 accuracy per fixed test set. Probe accuracies in
Table~\ref{tab:basin-order-correctness} are means and sample standard deviations
over five source-set selections.

\vspace{-0.1in}
\subsection{GSM8K Instruct--CoT Alignment and Conflict}
\vspace{-0.1in}
\label{app:gsm8k-instruct-cot-alignment}

We additionally test the alignment--conflict comparison under free-form
chain-of-thought generation on GSM8K~\citep{cobbe2021gsm8k}. We use
\texttt{Qwen/Qwen2.5-3B-Instruct} with its chat template and an instruction to
reason step by step before placing the final numerical answer in
\texttt{\textbackslash boxed\{\}}. For each question, $K=128$ continuations are
sampled at temperature $0.7$ with a maximum of 512 generated tokens. A fixed
numerical parser maps each continuation to its final boxed answer, and answer
counts provide the empirical basin masses.

Both training conditions use the same correctness labels. In the aligned
condition, the correct-answer basin is strictly larger than every incorrect
basin; one correct trajectory is paired with one trajectory from the smallest
incorrect basin. In the conflict condition, the correct-answer basin is
strictly smaller than every incorrect basin; one correct trajectory is paired
with one trajectory from the largest incorrect basin. The aligned pool
is downsampled to match the size of the conflict pool.
Each probe uses the complete layer-24 hidden state at the final numerical-answer
token, followed by feature standardization and $L_2$-regularized logistic
regression with $C=0.001$.

Evaluation uses the same 188 question-disjoint test questions for both probes.
For each test question, the probe scores all 128 sampled continuations; Top-1
is correct when the highest-scoring continuation gives the gold answer.
Table~\ref{tab:gsm8k-instruct-cot-alignment} shows that aligned training yields
higher held-out correctness than conflict training under otherwise identical
concept supervision.

\begin{table}[!htbp]
    \centering
    \caption{\textbf{GSM8K Instruct--CoT alignment and conflict.}
    Mean basin counts describe the paired training trajectories; both probes
    are evaluated on the same 188 held-out questions.}
    \label{tab:gsm8k-instruct-cot-alignment}
    \small
    \setlength{\tabcolsep}{5pt}
    \begin{tabular*}{\linewidth}{@{\extracolsep{\fill}}lrrr@{}}
        \toprule
        Training condition & Correct basin & Wrong basin & Top-1 (\%) \\
        \midrule
        Aligned  & $108.22$ & $1.00$ & $78.72$ \\
        Conflict & $26.22$  & $101.78$ & $70.74$ \\
        \bottomrule
    \end{tabular*}
\end{table}

\vspace{-0.1in}
\subsection{Fixed-Vector Steering and Basin Geometry}
\vspace{-0.1in}
\label{app:steering-geometry-protocol}

\paragraph{Label-defined option-text steering.}
For Figure~\ref{fig:steering-letter-transport}, each input is
\texttt{Question: \{question\}\textbackslash nAnswer: \{option\}}.
We extract the final option-token state and average the correct-minus-wrong
contrasts without normalization, using 100 aligned BBQ source questions
and 300 aligned SafetyBench source questions. Evaluation uses 100 aligned
and 100 conflict BBQ questions, and 300 of each condition for SafetyBench. The candidate-answer measure determines the sign of each correct--wrong mass gap. All original options are retained, and complete-option scores average token log-probabilities over the option text without EOS. Intervention is applied after blocks 19 and 16, respectively, at the final \texttt{Answer:} token and subsequent answer-prediction positions, for $\alpha=0,0.01,\ldots,0.10$. Pointwise 95\% intervals use 2,000
context-group bootstrap resamples \citep{efron1979bootstrap}.

\paragraph{Answer-letter geometry.}
For Figure~\ref{fig:basin-geometry}, we use Qwen2.5-7B base and 300 source
questions per dataset. We present both options, append each answer letter,
and extract block-20 states, averaging semantically matched contrasts over
AB/BA orders. The candidate measure averages binary choice probabilities
across AB/BA orders after normalization within each order.

\paragraph{Projection estimates.}
For Figure~\ref{fig:basin-geometry}, we form correct-minus-wrong contrasts
from the same block-20 answer-letter states and use the AB/BA-averaged
baseline probabilities.
Within each dataset, we split the 300 source questions into five folds by
context group. For each held-out fold, we average contrasts from the other
four folds, normalize the mean to unit length, and project the held-out
contrasts onto this direction. Thus, each question is scored using a
direction estimated without its context group. Pearson $r$ relates these
out-of-fold signed projections to basin log-ratios.

\paragraph{Mass-gap groups and contrast norms.}
We sort source questions by $\delta_q(a_T,a_O)$ and split them into three
equal-count groups of 100 questions each. For each group, we report the
mean norm of the individual contrasts and the norm of their mean.
The latter divided by the former measures norm-weighted directional
coherence.

\FloatBarrier

%% file: section/iclr/appendix_c_answer_letter_steering.tex
\vspace{-0.1in}
\section{Answer-Letter Responses to Option-Text Steering}
\vspace{-0.1in}
\label{app:complete-answer-steering}

We evaluate answer-letter logits on Gemma-3-4B-PT using the same BBQ and
SafetyBench question pairs and label-defined option-text vectors as in
Section~\ref{sec:steering-alignment}. We present all candidate options and
apply the vector at the final \texttt{Answer:} token, after block 19 for
BBQ and block 16 for SafetyBench. For each strength
$\alpha=0,0.01,\ldots,0.10$, we average semantically matched raw
answer-letter logit changes over all option permutations. This measures
the response of the answer-letter interface to the same intervention
direction used for complete-option scoring.

\begin{figure}[!htbp]
    \centering
    \includegraphics[width=\linewidth]{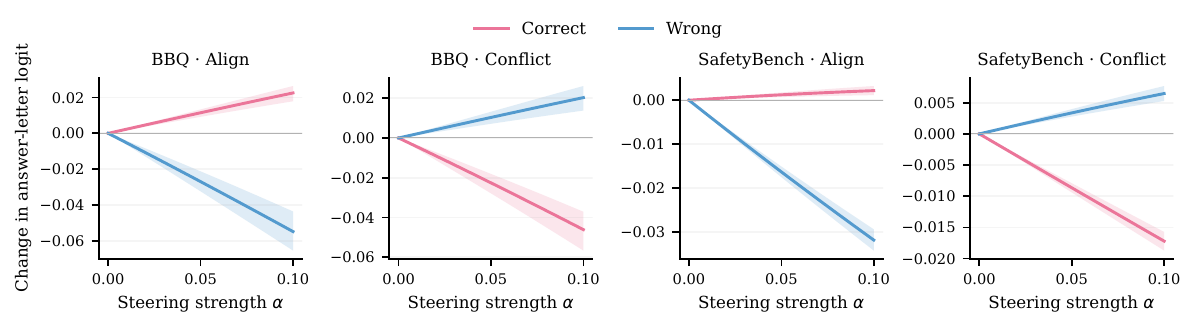}
    \caption{\textbf{Answer-letter responses to label-defined steering.}
    Changes in raw answer-letter logits on Gemma-3-4B-PT for the same vectors
    and question pairs as in Figure~\ref{fig:steering-letter-transport}.
    Shaded bands denote 95\% bootstrap intervals.}
    \label{fig:steering-relative-response}
\end{figure}

\FloatBarrier

%% file: section/iclr/appendix_e_basin_mass_geometry.tex
\section{Answer Geometry Across Basin-Mass Gaps}
\label{app:basin-gap-geometry}

\paragraph{Answer-pair selection for Figure~\ref{fig:basin-gap-separability}.}
Figure~\ref{fig:basin-gap-separability} examines the relationship between
answer-mass differences and representation geometry discussed in
Section~\ref{sec:abrh}. We use the candidate-answer measure
$\widetilde\mu_q$ defined in Section~\ref{sec:measuring-answer-basins}
and measure pairwise gaps through the log-mass contrast in
Eq.~\ref{eq:log-mass-contrast}. For Qwen2.5-7B on ARC-C, BBQ, and
SafetyBench, we select correct--wrong answer pairs from the same question
based on gap size and display groups with progressively smaller gaps
across five columns. This comparison shows how the separation of correct
and wrong answer representations varies with their mass difference.

\paragraph{Representations and PCA.}
For Figure~\ref{fig:basin-gap-separability}, we append a newline token to
each candidate answer and extract its hidden state at decoder block 23.
We subtract the mean representation of all candidates for the same
question and retain the original feature scales. For each dataset, we
fit a two-dimensional PCA using candidate representations from 100
questions. The displayed questions are disjoint from the PCA-fitting
questions; for BBQ and SafetyBench, their context groups are also
disjoint. All five columns within a row share the same PCA transformation
and axis limits.

\subsection{GSM8K Representation Geometry}
\label{app:gsm8k-representation-geometry}

\paragraph{Generation and answer grouping.}
Figures~\ref{fig:question-end-concentration} and
\ref{fig:gsm8k-answer-geometry} use Qwen2.5-7B generations on the GSM8K
test set to examine representation geometry associated with the answer
concentration and log-mass contrasts defined in Section~\ref{sec:abrh}.
We sample 32 responses per question at temperature 0.8 and top-$p$ 0.95,
with a maximum of 192 generated tokens. The prompt requests reasoning
followed by a numerical answer in the format \texttt{\#\#\#\# <answer>}.
We extract the number following this marker when available and otherwise
use the last parseable number in the response. After deduplicating
identical token sequences, we retain responses with extractable numerical
answers and group them by answer value. These group counts determine
answer frequencies, collision concentration, and log mass ratios.

\paragraph{Concentration before generation (Figure~\ref{fig:question-end-concentration}).}
We extract the hidden state at the final token of the generation prompt,
the colon in \texttt{Answer:}, from decoder block 20. We fit a
two-dimensional PCA on 100 questions and project the remaining 1,219
questions into the same space. We compute collision concentration using
the expression in Eq.~\ref{eq:empirical-question-concentration}, with the
number of retained responses for each question as the sample count, and
display the 200 questions with the highest and the 200 with the lowest
concentration. PCA retains the original feature scales, and its fitting
questions are disjoint from the displayed questions.

\paragraph{Answer-group representations after generation (Figure~\ref{fig:gsm8k-answer-geometry}).}
For each response, we extract the hidden state from decoder block 23 at
its last token after stripping trailing whitespace. Averaging these states
across responses with the same numerical answer yields an answer-group
representation. We subtract the mean of all answer-group representations
for the same question and fit a separate two-dimensional PCA on the same
100 fitting questions. Each question has equal total weight in the PCA
fit, and the original feature scales are retained. We select
correct--wrong answer pairs according to the log ratio of their answer
frequencies and display four groups with progressively smaller gaps.
Each point represents an answer group's mean hidden state, and $g$ is the
median gap among the questions in that column. All columns share the same
PCA transformation and axis limits.

\subsection{Additional Models}

Figures~\ref{fig:gap-qwen14}--\ref{fig:gap-gemma12} show correct--wrong
answer representations across decreasing estimated basin-mass gaps for
Qwen2.5-14B, Gemma-3-4B, and Gemma-3-12B.

\begin{figure}[!htbp]
    \centering
    \includegraphics[width=\textwidth]{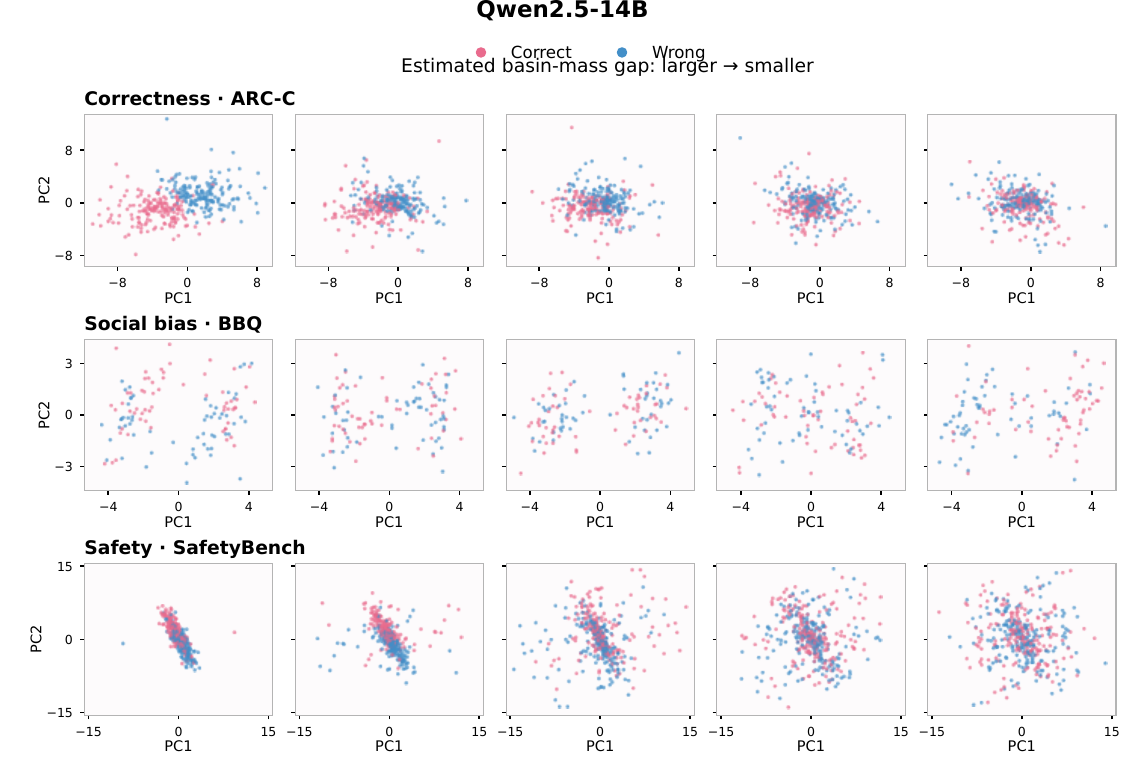}
    \caption{\textbf{Correct--wrong answer geometry across estimated basin-mass gaps in Qwen2.5-14B.}
    Rows cover correctness (ARC-C), social bias (BBQ), and safety (SafetyBench).
    Columns progress from larger to smaller gaps. Colors indicate dataset
    correctness; each row uses a shared PCA space.}
    \label{fig:gap-qwen14}
\end{figure}

\begin{figure}[!htbp]
    \centering
    \includegraphics[width=\textwidth]{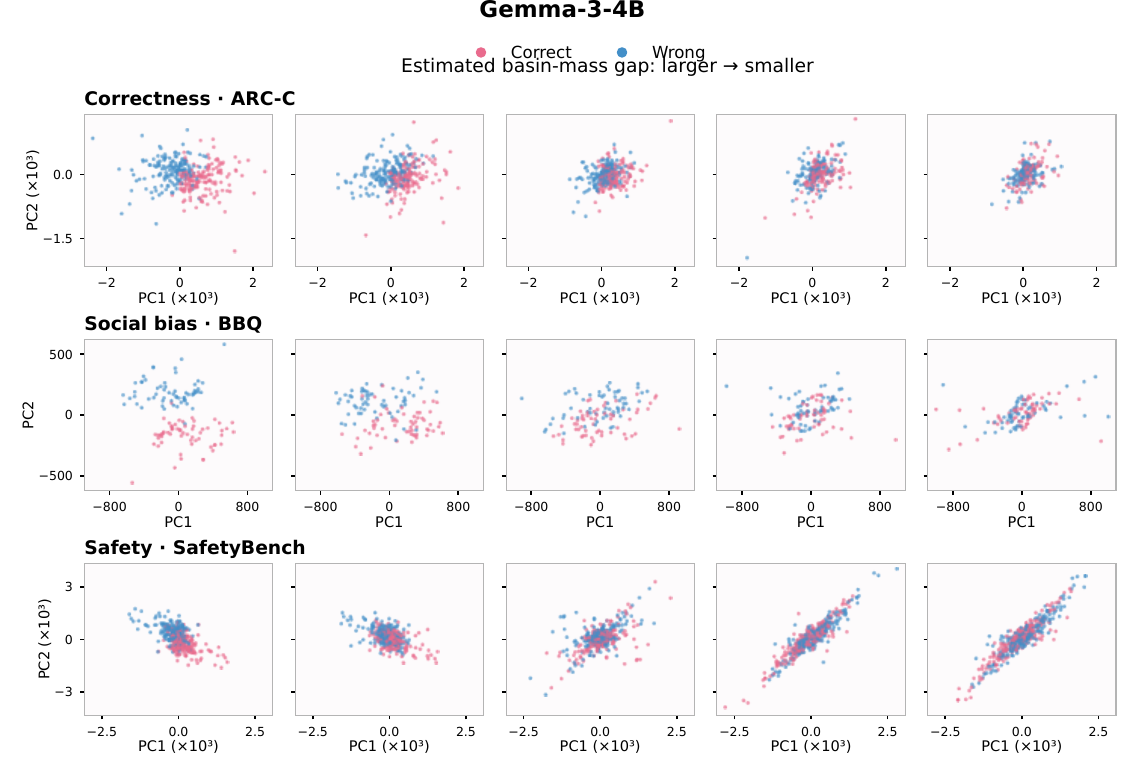}
    \caption{\textbf{Correct--wrong answer geometry across estimated basin-mass gaps in Gemma-3-4B.}
    Rows cover correctness (ARC-C), social bias (BBQ), and safety (SafetyBench).
    Columns progress from larger to smaller gaps. Colors indicate dataset
    correctness; each row uses a shared PCA space.}
    \label{fig:gap-gemma4}
\end{figure}

\begin{figure}[!htbp]
    \centering
    \includegraphics[width=\textwidth]{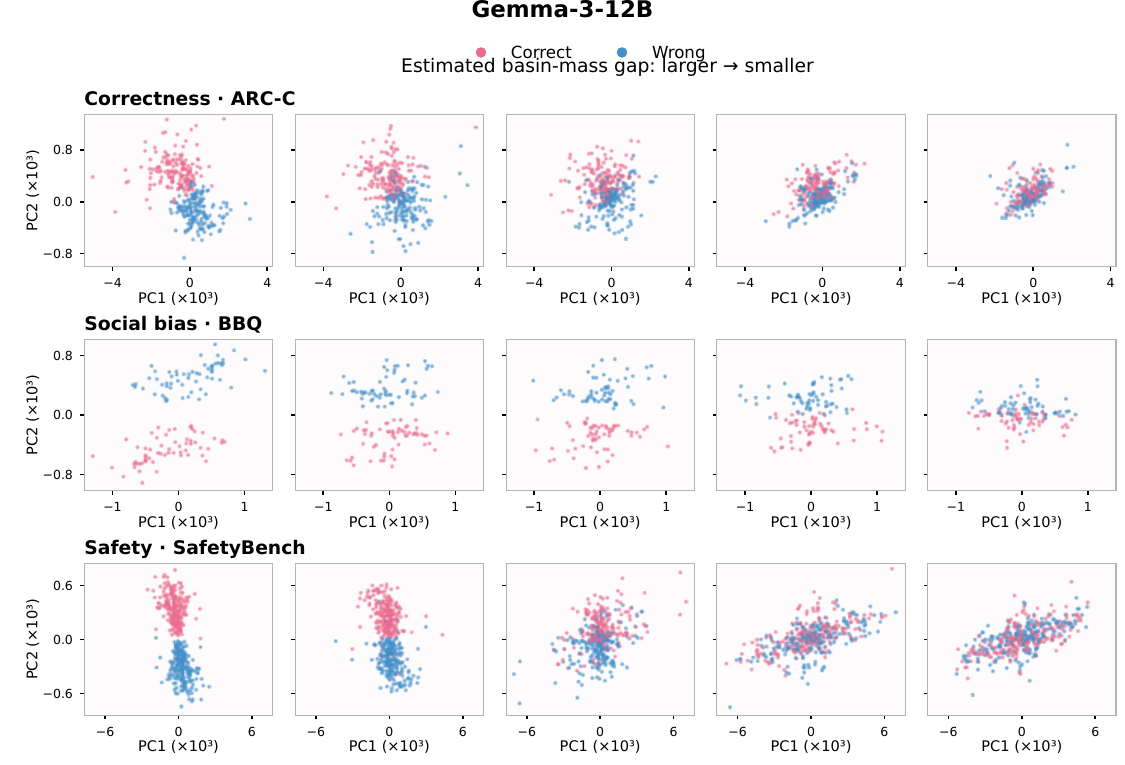}
    \caption{\textbf{Correct--wrong answer geometry across estimated basin-mass gaps in Gemma-3-12B.}
    Rows cover correctness (ARC-C), social bias (BBQ), and safety (SafetyBench).
    Columns progress from larger to smaller gaps. Colors indicate dataset
    correctness; each row uses a shared PCA space.}
    \label{fig:gap-gemma12}
\end{figure}

%% file: section/iclr/appendix_f_gsm8k_alignment.tex
\section{Additional GSM8K Alignment and Conflict Results}
\label{app:gsm8k-align-conflict}

Figure~\ref{fig:gsm8k-align-conflict} reports the GSM8K comparison across
four models. Align yields higher Top-1 accuracy than Conflict in all four
models. The displayed baseline lies between Align and Conflict for
Qwen2.5-7B, Qwen2.5-14B, and Gemma-3-12B; for Gemma-3-4B, Align and
Conflict both exceed the baseline.

\begin{figure}[!htbp]
    \centering
    \includegraphics[width=0.8\linewidth]{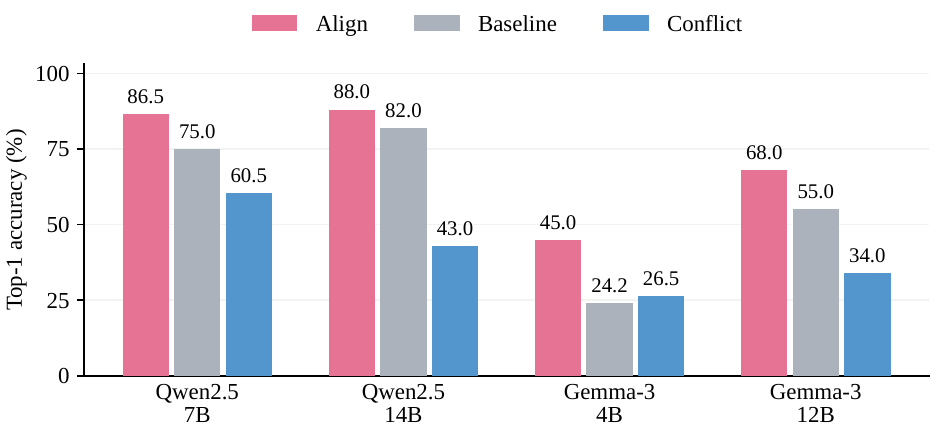}
    \caption{\textbf{GSM8K probing under label--mass alignment and conflict.}
    Held-out Top-1 accuracy of correctness probes trained under Align and
    Conflict, with the corresponding answer-selection baseline.
    Both probes retain the same correct--wrong labels and are evaluated
    on the same test questions.}
    \label{fig:gsm8k-align-conflict}
\end{figure}

%% file: section/iclr/appendix_g_safetybench_probing.tex
\section{Additional SafetyBench Basin-Mass Probing Results}
\label{app:safetybench-basin-probing}

Table~\ref{tab:safetybench-basin-order-correctness} reports the SafetyBench
results corresponding to Table~\ref{tab:basin-order-correctness}, using the
same metrics and training conditions.

\begin{table}[!htbp]
    \centering
    \caption{\textbf{Probing basin mass across training conditions on SafetyBench.}
    Larger training-pair mass gaps generally support more accurate basin
    readout and answer selection. Probes trained only on incorrect answers
    still recover basin-mass ordering on held-out questions.}
    \label{tab:safetybench-basin-order-correctness}
    \footnotesize
    \setlength{\tabcolsep}{3pt}
    \renewcommand{\arraystretch}{1.0}
    \resizebox{\linewidth}{!}{%
    \begin{tabular}{@{}clccccc@{}}
        \toprule
        \multirow{2}{*}{\textbf{Model}}
        & \multicolumn{1}{c}{\multirow{2}{*}{\textbf{Metric (\%)}}}
        & \multirow{2}{*}{\textbf{Mean logp}}
        & \multicolumn{4}{c}{\textbf{Basin-supervised linear probes}} \\
        \cmidrule(lr){4-7}
        & & & \textbf{Low} & \textbf{Medium} & \textbf{High} & \textbf{Wrong-only} \\
        \midrule
        \rowcolor{tableHeader}
        \multicolumn{7}{@{}l}{\strut\textbf{SafetyBench}} \\
        \addlinespace[1pt]
        \cellcolor{qwenSeven!8} & Basin ordering & --- & \cellcolor{qwenSeven!42.3}$64.87\,{\scriptstyle\pm\,1.87}$ & \cellcolor{qwenSeven!46.7}$67.03\,{\scriptstyle\pm\,1.83}$ & \cellcolor{qwenSeven!50.1}$68.67\,{\scriptstyle\pm\,2.03}$ & \cellcolor{qwenSeven!50.6}$68.89\,{\scriptstyle\pm\,1.94}$ \\
        \cellcolor{qwenSeven!8}\multirow{-2}{*}{Qwen2.5-7B base} & Top-1 correctness & \cellcolor{qwenSeven!24.6}\color{black}$55.00$ & \cellcolor{qwenSeven!22.6}$53.67\,{\scriptstyle\pm\,2.66}$ & \cellcolor{qwenSeven!30.8}$58.73\,{\scriptstyle\pm\,1.62}$ & \cellcolor{qwenSeven!38.1}$62.73\,{\scriptstyle\pm\,0.92}$ & \cellcolor{qwenSeven!35.5}$61.33\,{\scriptstyle\pm\,1.15}$ \\
        \addlinespace[1pt]
        \cellcolor{qwenFourteen!8} & Basin ordering & --- & \cellcolor{qwenFourteen!51.2}$69.18\,{\scriptstyle\pm\,1.77}$ & \cellcolor{qwenFourteen!52.8}$69.91\,{\scriptstyle\pm\,1.92}$ & \cellcolor{qwenFourteen!57.8}$72.12\,{\scriptstyle\pm\,2.11}$ & \cellcolor{qwenFourteen!54.4}$70.61\,{\scriptstyle\pm\,1.85}$ \\
        \cellcolor{qwenFourteen!8}\multirow{-2}{*}{Qwen2.5-14B base} & Top-1 correctness & \cellcolor{qwenFourteen!19.7}\color{black}$51.67$ & \cellcolor{qwenFourteen!25.9}$55.80\,{\scriptstyle\pm\,2.35}$ & \cellcolor{qwenFourteen!31.1}$58.93\,{\scriptstyle\pm\,2.23}$ & \cellcolor{qwenFourteen!34.6}$60.87\,{\scriptstyle\pm\,2.41}$ & \cellcolor{qwenFourteen!27.3}$56.67\,{\scriptstyle\pm\,2.14}$ \\
        \addlinespace[1pt]
        \cellcolor{gemmaFour!8} & Basin ordering & --- & \cellcolor{gemmaFour!33.8}$60.40\,{\scriptstyle\pm\,2.23}$ & \cellcolor{gemmaFour!42.8}\color{black}$65.13\,{\scriptstyle\pm\,5.56}$ & \cellcolor{gemmaFour!47.8}\color{black}$67.58\,{\scriptstyle\pm\,2.55}$ & \cellcolor{gemmaFour!61.0}\color{black}$73.52\,{\scriptstyle\pm\,2.06}$ \\
        \cellcolor{gemmaFour!8}\multirow{-2}{*}{Gemma-3-4B-PT} & Top-1 correctness & \cellcolor{gemmaFour!11.5}\color{black}$45.00$ & \cellcolor{gemmaFour!5.1}$35.53\,{\scriptstyle\pm\,0.69}$ & \cellcolor{gemmaFour!23.3}\color{black}$54.13\,{\scriptstyle\pm\,2.63}$ & \cellcolor{gemmaFour!25.6}\color{black}$55.60\,{\scriptstyle\pm\,1.09}$ & \cellcolor{gemmaFour!21.7}\color{black}$53.07\,{\scriptstyle\pm\,0.98}$ \\
        \addlinespace[1pt]
        \cellcolor{gemmaTwelve!8} & Basin ordering & --- & \cellcolor{gemmaTwelve!34.0}$60.53\,{\scriptstyle\pm\,2.79}$ & \cellcolor{gemmaTwelve!51.9}\color{black}$69.48\,{\scriptstyle\pm\,1.41}$ & \cellcolor{gemmaTwelve!44.9}\color{black}$66.19\,{\scriptstyle\pm\,1.59}$ & \cellcolor{gemmaTwelve!55.7}\color{black}$71.21\,{\scriptstyle\pm\,1.39}$ \\
        \cellcolor{gemmaTwelve!8}\multirow{-2}{*}{Gemma-3-12B-PT} & Top-1 correctness & \cellcolor{gemmaTwelve!19.7}\color{black}$51.67$ & \cellcolor{gemmaTwelve!9.5}$43.00\,{\scriptstyle\pm\,3.52}$ & \cellcolor{gemmaTwelve!36.2}\color{black}$61.73\,{\scriptstyle\pm\,2.90}$ & \cellcolor{gemmaTwelve!38.7}\color{black}$63.07\,{\scriptstyle\pm\,1.55}$ & \cellcolor{gemmaTwelve!45.9}\color{black}$66.67\,{\scriptstyle\pm\,2.17}$ \\
        \bottomrule
    \end{tabular}%
    }
\end{table}

%% file: section/iclr/appendix_d_related_work.tex
\vspace{-0.1in}
\section{Detailed Related Work}
\vspace{-0.1in}
\label{app:detailed-related-work}

\vspace{-0.1in}
\subsection{Detailed Related Work}
\vspace{-0.1in}

\paragraph{Concept Representations and Information Geometry.}
\citet{park2024linear} formalize linear concept representations through
counterfactual relations and connect readout and steering using a causal
inner product. Categorical and hierarchical concept geometry is studied by
\citet{park2025categorical}. A distinct question concerns the origin of
linearity: \citet{jiang2024origins} derive linear representations from the
next-token loss and gradient descent's implicit bias in a latent-variable
model. \citet{park2026information} study the information geometry induced
by softmax distributions. Their dual-steering analysis starts from a
concept variable with a linear probability readout and, under a
concept-factorization condition, characterizes control of the target
concept with minimal change to the off-target distribution.
The \textit{Answer-Basin Representation Hypothesis} (ABRH) places the answer measure at the center of the representation account.
Its statistics organize answer-related linear structure, while the
relationship between those statistics and external labels explains when
readout and intervention exhibit concept-consistent effects.

\paragraph{Answer Distributions and Semantic Uncertainty.}
Self-consistency combines sampled reasoning paths according to their final
answers \citep{wang2023selfconsistency}. Semantic entropy groups generations
by semantic equivalence and measures uncertainty over the resulting
meaning distribution \citep{kuhn2023semantic,farquhar2024semanticentropy}.
These constructions are related to answer measures through the aggregation
of text outputs into answer or meaning identities.
\citet{kossen2024semanticprobes} train linear probes to predict semantic
entropy from hidden states, including states before generation, for
computationally efficient uncertainty estimation and hallucination detection.
This provides a direct connection to the readability of distributional
statistics. ABRH gives the measure an explanatory role in the organization
of representations: concentration before generation and relative answer
mass after generation are linked aspects of the same theoretical object.
The framework uses this object to connect representation geometry,
concept-related readout, and intervention behavior.

\paragraph{Concept Probing and Activation Steering.}
Truthfulness probes examine information available in hidden states
\citep{azaria2023internal,marks2024geometry}; question-only probes predict
answer correctness before generation \citep{cencerrado2025noanswer}.
Analyses of disagreement between probe predictions and model outputs
examine the interpretation of such readouts \citep{liu2023cognitive}.
Activation steering uses directions obtained from attributes or contrastive
activations to change model outputs
\citep{turner2023activation,li2023inference,zou2023representation,rimsky2024caa},
including refusal behavior \citep{arditi2024refusal}.
These studies concern the predictive and causal uses of representation
directions. ABRH addresses their interpretation through the model's answer
measure: the relationship between labels and answer mass links the
statistical orientation of a direction to its concept-level effects.